\documentclass{article}

\usepackage{arxiv}

\usepackage[utf8]{inputenc} 
\usepackage[T1]{fontenc}    
\usepackage{hyperref}       
\usepackage{url}            
\usepackage{booktabs}       
\usepackage{amsfonts}       
\usepackage{nicefrac}       
\usepackage{microtype}      
\usepackage{lipsum}		
\usepackage{graphicx}
\usepackage{natbib}
\usepackage{doi}
 
\usepackage{bbm}
\usepackage{wrapfig}      
\usepackage{xcolor}      
\usepackage{caption}
\usepackage[american]{babel}
\usepackage{filecontents}
 
\usepackage{subfigure}
\usepackage{booktabs} 
\usepackage[normalem]{ulem}
 
\usepackage{tikz}
\usetikzlibrary{arrows.meta, automata,
                positioning,
                quotes}
 \usepackage{amsmath, amssymb}
\usepackage{pst-node}

\usepackage{algorithmic}
\usepackage{algorithm}
\usepackage{wrapfig}
\usepackage{lipsum}
\usepackage{verbatim}
\usepackage{amsmath}
\allowdisplaybreaks
\usepackage{amsthm}
\usepackage{microtype}
\usepackage{color}
\usepackage{eucal}
\usepackage{epsfig,amsmath,amssymb,amsfonts,amstext,amsthm,mathrsfs}
\usepackage{latexsym,graphics,epsf,epsfig,psfrag}
\usepackage{dsfont,color,epstopdf,fixmath}
\usepackage{enumitem}
\usepackage{algorithm,algorithmic,latexsym}

\newtheorem{assumption}{\textbf{Assumption}}

\newtheorem{corollary}{\textbf{Corollary}}
\newtheorem{lemma}{\textbf{Lemma}}
\newtheorem{theorem}{\textbf{Theorem}}
\newtheorem{prop}{\textbf{Proposition}}
\newtheorem{remark}{\textbf{Remark}}

\usepackage[acronym]{glossaries}

\newacronym{T1}{T1}{\theta_1}
\usepackage{color}

\newcommand{\mcs}{\mathcal{S}}
\newcommand{\mca}{\mathcal{A}}
\newcommand{\nn}{\nonumber}

\newcommand{\pit}{{\pi_{\theta}}}
\newcommand{\lse}{\text{LSE}}
\newcommand{\sv}{V_\sigma}
\newcommand{\sq}{Q_\sigma}
\newcommand{\epe}{\epsilon_{\text{est}}}
\newcommand{\pone}{\pi_{\theta_1}}
\newcommand{\ptwo}{\pi_{\theta_2}}
\usepackage{mathtools} 
\usepackage{siunitx} 
\usepackage{booktabs} 
\usepackage{tikz} 
\usepackage{cleveref}

\usepackage[toc,page,header]{appendix}
\usepackage{minitoc}

\title{Robust Constrained Reinforcement Learning}

\author{%
  Yue Wang  \\
  University at Buffalo\\
  Buffalo, NY 14228 \\
  \texttt{ywang294@buffalo.edu} \\
  \And
  Fei Miao  \\
  University of Connecticut\\
  Storrs, CT 06269 \\
  \texttt{fei.miao@uconn.edu} \\
  \And
  Shaofeng Zou \\
  University at Buffalo\\
  Buffalo, NY 14228 \\
  \texttt{szou3@buffalo.edu} \\
}  

\begin{document}

\maketitle

\begin{abstract}
Constrained reinforcement learning is to maximize the expected reward subject to constraints on utilities/costs. However, the training environment may not be the same as the test one, due to, e.g., modeling error, adversarial attack, non-stationarity, resulting in severe performance degradation and more importantly constraint violation. We propose a framework of robust constrained reinforcement learning under model uncertainty, where the MDP is not fixed but lies in some uncertainty set, the goal is to guarantee that constraints on utilities/costs are satisfied for all MDPs in the uncertainty set, and to maximize the worst-case reward performance over the uncertainty set. We design a robust primal-dual approach, and further theoretically develop guarantee on its convergence, complexity and robust feasibility. We then investigate a concrete example of $\delta$-contamination uncertainty set, design an online and model-free algorithm and theoretically characterize its sample complexity. 

\end{abstract}

\section{Introduction}

In many practical reinforcement learning (RL) applications,  it is critical for an agent to meet certain constraints on utilities and costs while maximizing the reward. However, in practice, it is often the case that the evaluation environment deviates from the training one, due to, e.g., modeling error of the simulator, adversarial attack, and non-stationarity. This could lead to a significant performance degradation in reward, and more importantly, constraints may not be satisfied anymore, which is severe in safety-critical applications. For example, a drone may run out of battery due to model deviation between the training and test environments, resulting in a crash.
To solve these issues,  we propose a framework of  robust constrained RL under model uncertainty. Specifically, the Markov decision process (MDP) is not fixed and lies in an uncertainty set \cite{nilim2004robustness,iyengar2005robust,bagnell2001solving}, and the goal is to maximize the worst-case accumulative discounted reward over the uncertainty set while guaranteeing that constraints are satisfied for all MDPs in the uncertainty set at the same time. 

Despite of its practical importance, studies on the problem of robust constrained RL are limited in the literature. 
Two closely related topics are robust RL \cite{bagnell2001solving,nilim2004robustness,iyengar2005robust} and constrained RL \cite{altman1999constrained}. The problem of constrained RL  \cite{altman1999constrained} aims to find a policy that optimizes an objective reward while satisfying certain constraints on costs/utilities. For the problem of robust RL \cite{bagnell2001solving,nilim2004robustness,iyengar2005robust}, the MDP  is not fixed but lies in some uncertainty set, and the goal is to find a policy that optimizes the robust value function, which measures the worst-case accumulative reward over the uncertainty set. 
The problem of robust constrained RL was investigated in \cite{russel2020robust,mankowitz2020robust}, where two heuristic approaches were proposed. The basic idea in \cite{russel2020robust,mankowitz2020robust} is to first evaluate the worst-case performance of the policy over the uncertainty set, and then use that together with classical policy improvement methods, e.g., policy gradient \cite{sutton1999policy}, to update the policy. However, as will be discussed in more details later, these approaches may not necessarily lead to an improved policy, and thus may not perform well in practice. 

In this paper, we design the robust primal-dual algorithm for the problem of robust constrained RL.  Our approach employs the true gradient of the Lagrangian function, which is the weighted sum of two robust value functions, instead of approximating the gradient heuristically as in \cite{russel2020robust}. We theoretically characterize the convergence and complexity of our robust primal-dual method, and prove the robust feasibility of our solution for all MDPs in the uncertainty set. 
We further present a concrete example of $\delta$-contamination uncertainty set \cite{hub65,du2018robust,huber2009robust,nishimura2004search,Kiyohiko2006,prasad2020learning,prasad2020robust,wang2021online,wang2022policy}, for which we extend our algorithm to the online and model-free setting, and theoretically characterize its finite-time error bound. {To the best of the authors' knowledge, our work is the first in the literature of robust constrained RL that comes with model-free algorithms, theoretical convergence guarantee, complexity analyses, and robust feasibility guarantee.}
In particular, the technical challenges and our major contributions are summarized as follows.
\begin{itemize}[leftmargin=*]

    

    \item {In the non-robust setting, the sum of two value functions is actually a value function of the combined reward. However, this does not hold in the robust setting, since the worst-case transition kernels for the two robust value functions are not necessarily the same. 
    Therefore, the geometry of our Lagrangian function is much more complicated. In this paper, we formulate the dual problem of the robust constrained RL problem as a minimax linear-nonconcave optimization problem, and show that the optimal dual variable is bounded. We then construct a robust primal-dual algorithm by alternatively updating the primal and dual variables. We theoretically prove the convergence to stationary points, and characterize its complexity. }
    \item {In general, convergence to stationary points of the Lagrangian function does not necessarily imply that the solution is feasible. We design a novel proof to show that the gradient belongs to the normal cone of the feasible set, based on which we further prove the feasibility of the obtained policy.}
    
    \item {Based on existing literature on constrained MDP \cite{ding2020natural,ding2021provably,li2021faster,liu2021fast,ying2021dual} and robust MDP \cite{wang2022policy}, at first, we expect that the robust constrained RL also has zero duality gap, and further global optimum can be achieved. However, this is not necessarily true. Note that the set of visitation distribution being convex is one key property to show zero duality gap of constrained MDP \cite{altman1999constrained,paternain2019constrained}. In this paper, we constructed a novel counter example showing that the set of robust visitation distributions for our robust problem is non-convex. }
        \item {We further apply and extend our results on an important uncertainty set referred as $\delta$-contamination model \cite{hub65}.  Under this model, the robust value functions are not differentiable and we hence propose a smoothed approximation of the robust value function towards a better geometry. We further investigate the practical online and model-free setting and design an actor-critic type algorithm. We also establish its convergence, sample complexity, and robust feasibility.}
\end{itemize}

We then discuss works related to robust constrained RL. 

\textbf{Robust constrained RL.} In \cite{russel2020robust}, the robust constrained RL problem was studied, and a heuristic approach was developed. The basic idea is to estimate the robust value functions, and then to use the vanilla policy gradient method \cite{sutton1999policy} with the vanilla value function replaced by the robust value function. However, this approach did not take into consideration the fact that the worst-case transition kernel is also a function of the policy (see Section 3.1 in \cite{russel2020robust}), and therefore the "gradient" therein is not actually the gradient of the robust value function. Thus, its performance and convergence cannot be theoretically guaranteed.
The other work \cite{mankowitz2020robust} studied the same robust constrained RL problem under the continuous control setting, and proposed a similar heuristic algorithm. They first proposed a robust Bellman operator and used it to estimate the robust value function, which is further combined with some non-robust continuous control algorithm to update the policy. Both approaches in \cite{russel2020robust} and \cite{mankowitz2020robust} inherit the heuristic structure of "robust policy evaluation" + "non-robust vanilla policy improvement", which may not necessarily guarantee an improved policy in general. In this paper, we employ a "robust policy evaluation" + "\textbf{\textit{robust}} policy improvement" approach, which guarantees an improvement in the policy, and more importantly, we provide theoretical convergence guarantee, robust feasibility guarantee, and complexity analysis for our algorithms.

\textbf{Constrained RL.} 
The most commonly used method for constrained RL is the primal-dual method \cite{altman1999constrained,paternain2019constrained,paternain2022safe,liang2018accelerated,stooke2020responsive,tessler2018reward,yu2019convergent,zheng2020constrained,efroni2020exploration,auer2008near}, 
which augments the objective with a sum of constraints
weighted by their corresponding Lagrange multipliers, and then alternatively updates the primal and dual variables. It was shown  that the strong duality holds for constrained RL, and hence the primal-dual method has zero duality gap \cite{paternain2019constrained,altman1999constrained}. The convergence rate of the primal-dual method was investigated in \cite{ding2020natural,ding2021provably,li2021faster,liu2021fast,ying2021dual}.
Another class of method is the primal method, which is to enforce the constraints without resorting to the Lagrangian formulation \cite{achiam2017constrained,liu2020ipo,chow2018lyapunov,dalal2018safe,xu2021crpo,yang2020projection}. 
The above studies, when directly applied to \textit{robust} constrained RL, cannot guarantee the constraints when there is model deviation. Moreover, the objective and constraints in this paper take min over the uncertainty set (see \eqref{eq:value}), and therefore have much more complicated geometry than the non-robust case.

\textbf{Robust RL under model uncertainty.} 
Model-based robust RL was firstly introduced and studied in \cite{iyengar2005robust,nilim2004robustness,bagnell2001solving,satia1973markovian,wiesemann2013robust,lim2019kernel,xu2010distributionally,yu2015distributionally,lim2013reinforcement,tamar2014scaling}, where the uncertainty set is assumed to be known, and the problem can be solved using robust dynamic programming. It was then extended to the model-free setting, where the uncertainty set is unknown, and only samples from its centroid can be collected \cite{roy2017reinforcement,wang2021online,wang2022policy,zhou2021finite,yang2021towards,panaganti2021sample,ho2018fast,ho2021partial}. There are also empirical studies on robust RL, e.g., \cite{vinitsky2020robust,pinto2017robust,abdullah2019wasserstein,hou2020robust,rajeswaran2017epopt,huang2017adversarial,kos2017delving,lin2017tactics,pattanaik2018robust,mandlekar2017adversarially}. These works focus on robust RL without constraints, whereas in this paper we investigate robust RL with constraints, which is more challenging. {There is a related line of works on (robust) imitation learning \cite{ho2016generative,fu2017learning,torabi2018generative,viano2022robust}, which can be formulated as a constrained problem. But their problem settings and approaches are fundamentally different from ours.}


\color{black}

\section{Preliminaries}
\textbf{Constrained MDP.}
A constrained MDP (CMDP) can be specified by a tuple $(\mathcal{S},\mathcal{A},  \mathsf P, r, c_1,...,c_m, \gamma)$, where $\mcs$ and $\mca$ denote the state and action spaces, $\mathsf P=\left\{p^a_s \in \Delta_{\mcs}, a\in\mca, s\in\mcs\right\}$ is the transition kernel\footnote{$\Delta_{\mcs}$ denotes the probability simplex supported on $\mcs$.}, $r: \mcs\times\mca \to [0,1]$ is the reward function, $c_i: \mcs\times\mca \to [0,1], i=1,...,m$ are utility functions in the constraint, and $\gamma\in[0,1)$ is the discount factor. A stationary policy $\pi$ is a mapping $\pi: \mcs\to\Delta_{\mca}$, where $\pi(a|s)$ denotes the probability of taking action $a$ when the agent is at state $s$. The set of all the stationary policies is denoted by $\Pi$.

The non-robust value function of reward $r$ and a policy $\pi$ is defined as the expected accumulative discounted reward if the agent follows policy $\pi$: $\mathbb{E}_{\pi,\mathsf P}[\sum^\infty_{t=0} \gamma^t r(S_t,A_t)|S_0=s]$, where $\mathbb{E}_{\pi,\mathsf P}$ denotes the expectation  when the policy is $\pi$ and the transition kernel is $\mathsf P$. Similarly, the non-robust value function of $c$ is defined as $\mathbb{E}_{\pi,\mathsf P}[\sum^\infty_{t=0} \gamma^t c_i(S_t,A_t)|S_0=s]$.
The goal of CMDP is to find a policy that maximizes the expected reward subject to constraints on the expected utility:
\begin{align}\label{eq:CMDP}
    \max_{\pi\in\Pi} \mathbb{E}_{\pi,\mathsf P}\left[\sum^\infty_{t=0} \gamma^t r(S_t,A_t)|S_0\sim\rho\right], \text{ s.t. } \mathbb{E}_{\pi,\mathsf P}\left[\sum^\infty_{t=0} \gamma^t c_i(S_t,A_t)|S_0\sim\rho\right]\geq b_i, 1\leq i\leq m,
\end{align}
where $b_i$'s are some positive thresholds {and $\rho$ is the initial state distribution.}


Define the visitation distribution induced by policy $\pi$ and  transition kernel $\mathsf P$: $d_{\rho,\mathsf P}^\pi(s,a)=(1-\gamma)\sum^\infty_{t=0} \gamma^t \mathbb{P}(S_t=s,A_t=a|S_0\sim\rho,\pi,\mathsf P)$. It can be shown that the set of the visitation distributions of all policies $\{d_{\rho,\mathsf P}^\pi\in\Delta_{\mcs\times\mca}:\pi\in\Pi \}$ is convex  \cite{paternain2022safe,altman1999constrained}.
A standard assumption in the literature is the Slater's condition \cite{bertsekas2014constrained,ding2021provably}: There exists a constant $\zeta>0$ and a policy $\pi\in\Pi$ s.t. $\forall i$, 
$
 \mathbb{E}_{\pi,\mathsf P}\left[\sum^\infty_{t=0} \gamma^t c_i(S_t,A_t)|S_0\sim\rho\right]-b_i\geq \zeta.
$
Based on the convexity of the set of all visitation distributions and Slater's condition, strong duality can be established \cite{altman1999constrained,paternain2019constrained}.

\textbf{Robust MDP.}
A robust MDP can be specified by a tuple  $(\mathcal{S},\mathcal{A},  \mathcal P, r, \gamma)$. In this paper, we focus on the $(s,a)$-rectangular uncertainty set \cite{nilim2004robustness,iyengar2005robust}, i.e., $\mathcal{P}=\bigotimes_{s,a} \mathcal{P}^a_s$, where $\mathcal{P}^a_s \subseteq \Delta_{\mcs}$. 
Denote the transition kernel at time $t$ by $\mathsf P_t$, and let $\kappa=(\mathsf P_0,\mathsf P_1,...)$ be the dynamic model, where $\mathsf P_t\in\mathcal{P}, \forall t\geq 0$. We then define the robust value function of a policy $\pi$ as the worst-case expected accumulative discounted reward following policy $\pi$ over all MDPs in the uncertainty set \cite{nilim2004robustness,iyengar2005robust}:
\begin{align}\label{eq:value}
    V^{\pi}_r(s)\triangleq\min_{\kappa\in\bigotimes_{t\geq 0}\mathcal{P}} \mathbb{E}_{ \kappa}\left[\sum_{t=0}^{\infty}\gamma^t   r(S_t,A_t )|S_0=s,\pi\right],
\end{align}
where $\mathbb{E}_{\kappa}$ denotes the expectation when the state transits according to $\kappa$. It has been shown that the robust value function is the fixed point of the robust Bellman operator \cite{nilim2004robustness,iyengar2005robust,puterman2014Markov}:
$
    \mathbf T_\pi V(s)\triangleq \sum_{a\in\mca} \pi(a|s) \left(r(s,a)+\gamma \sigma_{\mathcal{P}^a_s}(V) \right),
$
where $\sigma_{\mathcal{P}_s^a}(V)\triangleq\min_{p\in\mathcal{P}_s^a} p^\top V$ is the support function of $V$ on $\mathcal{P}_s^a$. Similarly, we can define the robust action-value function for a policy $\pi$:
$
    Q_r^{\pi}(s,a)=\min_{\kappa }\mathbb{E}_{ \kappa}\left[\sum_{t=0}^{\infty}\gamma^t  r(S_t,A_t )|S_0=s,A_0=a,\pi\right].
$

Note that the minimizer of \eqref{eq:value}, $\kappa^*$, is stationary in time \cite{iyengar2005robust}, which we denote by $\kappa^*=\{\mathsf P^\pi,\mathsf P^\pi,...\}$, and refer to $\mathsf P^\pi$ as the  worst-case transition kernel. Then the robust value function $V^\pi_r$ is actually the value function under policy $\pi$ and transition kernel $\mathsf P^\pi$.
The goal of robust RL is to find the optimal robust policy $\pi^*$  that maximizes the worst-case accumulative discounted reward:
$
\pi^*=\arg\max_{\pi} V_r^{\pi}(s), \forall s\in\mcs.
$ 

\section{Robust Constrained  RL}\label{sec:CRMDP}


{The motivation for the problem of robust constrained MDP has two folds. The first is to guarantee that the constraints are always satisfied even if there is a mismatch between the training and evaluation environments. The second one is that among those feasible policies, we want to find one that optimizes the worst-case reward performance in the uncertainty set. In the following, we formulate the robust constrained problem.}
%
%
\begin{align}\label{eq:pCRMDP}
    \max_{\theta\in\Theta} V_{r}^{\pi_\theta}(\rho), \text{ s.t. } V_{c_i}^{\pi_\theta}(\rho)\geq b_i, 1\leq i\leq m,
\end{align}
 {where $V^{\pi_\theta}_{c_i}(\rho)$ and $V^{\pi_\theta}_{r}(\rho)$  are the robust value function of $c_i$ and $r$ under $\pi_\theta$.}

{Note that the goal of \eqref{eq:pCRMDP} is to find a policy that maximizes the robust reward value function among those policies satisfying that their robust utility value functions are above given thresholds. Any feasible solution to \eqref{eq:pCRMDP} can guarantee that under any MDP in the uncertainty set, its accumulative discounted utility is always no less than $b_i$, which guarantees robustness to constraint violation under model uncertainty. Furthermore, the optimal solution to \eqref{eq:pCRMDP} achieves the best "worst-case reward performance" among all feasible solutions. If we use the optimal solution to\eqref{eq:pCRMDP}, then under any MDP in the uncertainty set, we have a guaranteed reward no less than the value of \eqref{eq:pCRMDP}. This ensures that the solution of \eqref{eq:pCRMDP} is the best "worst-case reward performance" among all the feasible policies. }

In this paper, we focus on a general parameterized policy class, i.e., $\pit\in\Pi_\Theta$, where $\Theta\subseteq \mathbb R^d$ is a parameter set and $\Pi_\Theta$ is a class of parameterized policies, e.g., direct parameterized policy, softmax or neural network policy. {Many policy class has enough representative power such that the whole policy space $\Pi=\Pi_\Theta$, hence it is equivalent to consider the parameterized policies. }

For technique convenience, we adopt a standard assumption on the policy class.
\begin{assumption}
The policy class $\Pi_\Theta$ is $k$-Lipschitz and $l$-smooth, i.e., for any $s\in\mcs$ and $a\in\mca$ and for any $\theta\in\Theta$, there exist universal constants $k, l$, such that $
    \|\nabla \pit(a|s)\| \leq k,
$
  and 
$
     \|\nabla^2 \pit(a|s)\|\leq l$.
\end{assumption}
 This assumption  can be easily satisfied by many policy classes, e.g., direct parameterization \cite{agarwal2021theory}, soft-max \cite{mei2020global,li2021softmax,zhang2021global,wang2020finite}, or neural network with Lipschitz and smooth activation functions \cite{du2019gradient,neyshabur2017implicit,miyato2018spectral}.

Similar to the non-robust case, the problem \eqref{eq:pCRMDP} is equivalent to the following max-min problem:
\begin{align}\label{eq:l_CRMDP}
    \max_{\theta\in\Theta}\min_{\lambda_i\geq 0} V_{r}^{\pi_\theta}(\rho)+\sum^m_{i=1}\lambda_i(V_{c_i}^{\pi_\theta}(\rho)-b_i).
\end{align}

Unlike non-robust CMDP, strong duality for robust constrained RL may not hold. For robust RL, the robust value function can be viewed as the value function for policy $\pi$ under its worst-case transition kernel $\mathsf P^\pi$, and therefore can be written as the inner product between the reward (utility) function and the visitation distribution induced by $\pi$ and $\mathsf P^\pi$ (referred to as robust visitation distribution of $\pi$). The following lemma shows that the set of robust visitation distributions may not be convex, and therefore, the approach used in \cite{altman1999constrained,paternain2019constrained} to show strong duality cannot be applied here.
\begin{lemma}\label{lemma:fail duality}
There exists a robust MDP, such that the set of robust visitation distributions
is non-convex.
\end{lemma}


In the following, we focus on the dual problem of \eqref{eq:l_CRMDP}. Due to the weak duality, the optimal solution of the dual problem is a sub-optimal solution of the \eqref{eq:l_CRMDP}. 
For simplicity, we investigate the case with one constraint, and extension to the case with multiple constraints is straightforward: 
\begin{align}\label{eq:dual}
    \min_{\lambda\geq 0} \max_{\theta\in\Theta} V^{\pi_\theta}_r(\rho)+\lambda(V_c^{\pi_\theta}(\rho)-b).
\end{align}

We make an assumption of Slater's condition, assuming there exists at least one strictly feasible policy \cite{bertsekas2014constrained,ding2021provably}, under which, we further show that the optimal dual variable of \eqref{eq:dual} is bounded.
\begin{assumption}\label{ass:slater}
There exists $\zeta>0$ and a policy $\pi\in\Pi_\Theta$,  s.t. $V_{c}^{\pi}(\rho)-b\geq \zeta$.
\end{assumption}

\begin{lemma}\label{lemma:lam_bound}
Denote the optimal solution of \eqref{eq:dual} by $(\lambda^*,\pi_{\theta^*})$. Then, $\lambda^*\in\left[0,\frac{2}{\zeta(1-\gamma)} \right]$.
\end{lemma}
Lemma \ref{lemma:lam_bound} suggests that the dual problem \eqref{eq:dual} is equivalent to a bounded min-max problem: 
\begin{align}\label{eq:bounded dual}
     \min_{\lambda\in\left[0,\frac{2}{\zeta(1-\gamma)}\right]} \max_{\theta\in\Theta} V^{\pi_\theta}_r(\rho)+\lambda(V_c^{\pi_\theta}(\rho)-b).
\end{align}

The problem \eqref{eq:bounded dual} is a bounded linear-nonconcave optimization problem. We then propose our robust primal-dual algorithm for robust constrained RL in Algorithm \ref{alg:rpd}. 
\begin{algorithm}[H]
\caption{Robust Primal-Dual algorithm (RPD)}
\label{alg:rpd}
\textbf{Input}:   $T$,  $\alpha_t$, $\beta_t$, $b_t$\\
\textbf{Initialization}: $\lambda_0$, $\theta_0$\
\begin{algorithmic}
\FOR {$t=0,1,...,T-1$}
\STATE {$\lambda_{t+1}\leftarrow \mathbf{\prod}_{[0,\Lambda^*]} \left(\lambda_t -\frac{1}{\beta_t} \left(V^{\pi_{\theta_t}}_{c}(\rho)-b \right)-\frac{b_t}{\beta_t}\lambda_t  \right)$}
\STATE {$\theta_{t+1}\leftarrow \mathbf{\prod}_\Theta \left(\theta_t+\frac{1}{\alpha_t} \left(\nabla_{\theta}V^{\pi_{\theta_t}}_{r}(\rho)+\lambda_{t+1}\nabla_\theta V^{\pi_{\theta_t}}_{c}(\rho)  \right) \right)$}
\ENDFOR
\end{algorithmic}
\textbf{Output}: $\theta_T$
\end{algorithm}
  The basic idea of Algorithm \ref{alg:rpd} is to perform gradient descent-ascent w.r.t. $\lambda$ and  $\theta$ alternatively. When the policy $\pi$ violates the constraint, the dual variable $\lambda$ increases such that $\lambda V^\pi_{c}$ dominates  $ V^\pi_{r}$. Then the gradient ascent will update $\theta$ until the policy satisfies the constraint. Therefore, this approach is expected to find a feasible policy (as will be shown in Lemma \ref{lemma:222}).

Here, $\mathbf{\prod}_{\mathcal{X}}(x)$ denotes the projection of $x$ to the set $\mathcal{X}$, and $\left\{b_t\right\}$ is a non-negative monotone decreasing sequence, which will be specified later.  Algorithm \ref{alg:rpd} reduces to the vanilla gradient descent-ascent algorithm in \cite{lin2020gradient} if  $b_t=0$. However, $b_t$ is critical to the convergence of Algorithm \ref{alg:rpd} \cite{xu2020unified}. The outer problem of \eqref{eq:bounded dual} is actually linear, and after introducing  $b_t$, the update of $\lambda_t$ can be viewed as a gradient descent of a strongly-convex function $\lambda(V_c-b)+\frac{b_t}{2}\lambda^2$, which converges more stable and faster.

Denote that Lagrangian function by $V^L(\theta,\lambda)\triangleq V^{\pit}_{r}(\rho)+\lambda(V_{c}^{\pit}(\rho)-b)$, and further denote the gradient mapping of Algorithm \ref{alg:rpd} by
\begin{align}\label{eq:Gt}
G_t\triangleq\left[
\begin{array}{lr}
     &\beta_t \left(\lambda_t-\mathbf{\prod}_{[0,\Lambda^*]} \left(\lambda_t -\frac{1}{\beta_t} \left(\nabla_\lambda V^L (\theta_t,\lambda_t) \right)\right)\right)  \\
     & \alpha_t\left(\theta_t-\mathbf{\prod}_\Theta \left(\theta_t+\frac{1}{\alpha_t} \left(\nabla_\theta V^L (\theta_t,\lambda_t) \right) \right)\right) 
\end{array}
\right].
\end{align}
The gradient mapping is a standard measure of convergence for projected optimization approaches \cite{beck2017first}. Intuitively, it reduces to the gradient $(\nabla_\lambda V^L, \nabla_\theta V^L)$, when  $\Lambda^*=\infty$ and $\Theta=\mathbb{R}^d$, and it measures the updates of $\theta$ and $\lambda$ at time step $t$. If $\|G_t\|\to 0$,  the updates of both variables are small, and hence the algorithm converges to a stationary solution. 

To show the convergence of Algorithm \ref{alg:rpd}, we make the following Lipschitz smoothness assumption.
\begin{assumption}\label{assumption:lipschitz}
The gradients of the Lagrangian function are Lipschitz: \begin{align}
    \|\nabla_{\lambda} V^L(\theta,\lambda)|_{\theta_1}-\nabla_{\lambda} V^L(\theta,\lambda)|_{\theta_2} \|&\leq  L_{11} \|\theta_1-\theta_2\|,\\
     \|\nabla_{\lambda} V^L(\theta,\lambda)|_{\lambda_1}-\nabla_{\lambda} V^L(\theta,\lambda)|_{\lambda_2} \|&\leq L_{12}|\lambda_1-\lambda_2|,\\
    \|\nabla_{\theta} V^L(\theta,\lambda)|_{\theta_1}-\nabla_{\theta} V^L(\theta,\lambda)|_{\theta_2} \|&\leq L_{21}\|\theta_1-\theta_2\|,\\
     \|\nabla_{\theta} V^L(\theta,\lambda)|_{\lambda_1}-\nabla_{\theta} V^L(\theta,\lambda)|_{\lambda_2} \|&\leq L_{22} |\lambda_1-\lambda_2|.
\end{align}
\end{assumption}

In general, Assumption \ref{assumption:lipschitz} may or may not hold depending on the uncertainty set model. As will be shown in Section \ref{sec:ex}, even if Assumption \ref{assumption:lipschitz} does not hold, we can design a smoothed approximation of the robust value function, so that the assumption holds for the smoothed problem.

In the following theorem, we show that our robust primal-dual algorithm  converges to a stationary point of the min-max problem \eqref{eq:boundedsmoothed}, with a complexity of $\mathcal{O}(\epsilon^{-4})$.
\begin{theorem}\label{thm:1}
Under Assumption \ref{assumption:lipschitz}, if we 
 set step sizes $\alpha_t,\beta_t,$ and $b_t$ as in Section \ref{sec:constants} and $T=\mathcal{O}(\epsilon^{-4})$, then 
$
    \min_{1\leq t\leq T}\|G_t\|\leq 2\epsilon.
$
\end{theorem}




The next proposition characterizes the feasibility of the obtained policy.
\begin{prop}\label{prop:feasible}
 Denote by $W\triangleq \arg\min_{1\leq t\leq T} \| G_t\|$.
 If   $\lambda_W -\frac{1}{\beta_W} \left(\nabla_\lambda V^L_\sigma (\theta_W,\lambda_W) \right)\in [0,\Lambda^*)$, then $\pi_W$ satisfies the constraint with a $2\epsilon$-violation.
\end{prop}
{In general,  convergence to stationary points of the Lagrangian function does not necessarily imply that the solution is feasible. Proposition \ref{prop:feasible} shows that Algorithm \ref{alg:rpd} always return a policy that is robust feasible, i.e., satisfying the constraints in \eqref{eq:pCRMDP}.   
}
Intuitively, if we set $\Lambda^*$ larger so that the optimal solution $\lambda^*\in[0,\Lambda^*)$, then Algorithm \ref{alg:rpd} is expected to converge to an interior point of $[0,\Lambda^*]$ and therefore, $\pi_W$ is feasible. 
On the other hand, $\Lambda^*$ can't be set too large. Note that the complexity in Theorem \ref{thm:1} depends on $\Lambda^*$ (see \eqref{eq:T} in the appendix), and a larger $\Lambda^*$ means a higher complexity.

\vspace{-0.3cm}
\section{$\delta$-Contamination Uncertainty Set}\label{sec:ex}
In this section, we investigate a concrete example of robust constrained RL with $\delta$-contamination uncertainty set. The method we developed here can be similarly extended to other type of uncertainty sets like KL-divergence or total variation. 
The $\delta$-contamination uncertainty set models the scenario where the state transition of the MDP could be arbitrarily perturbed with a small probability $\delta$. This model is widely used to model distributional uncertainty in the literature of robust learning and optimization, e.g., \cite{hub65,du2018robust,huber2009robust,nishimura2004search,Kiyohiko2006,prasad2020learning,prasad2020robust,wang2021online,wang2022policy}. Specifically, let $\mathsf P=\left\{p^a_s|s\in\mcs,a\in\mca \right\}$ be the centroid transition kernel, then the $\delta$-contamination uncertainty set centered at $\mathsf P$ is defined as $\mathcal{P}\triangleq\bigotimes_{s\in\mcs,a\in\mca}\mathcal{P}^a_s$, where
$
    \mathcal{P}^a_s\triangleq \left\{(1-\delta )p^a_s+\delta q|q\in\Delta_{\mcs} \right\},  s\in\mcs, a\in\mca.
$

Under the $\delta$-contamination setting, the robust Bellman operator  can be explicitly computed:
$
    \mathbf T_\pi V(s)= \sum_{a\in\mca} \pi(a|s) \left(r(s,a)+\gamma \left(\delta \min_{s'} V(s')+(1-\delta)\sum_{s'\in\mcs}p^a_{s,{s'}}V({s'})\right) \right).
$
In this case, the robust value function is non-differentiable due to the $\min$ term, and hence Assumption \ref{assumption:lipschitz} does not hold. 
One possible approach is to use sub-gradient  \cite{clarke1990optimization,wang2022policy}, which, however, is less stable, and its convergence is difficult to characterize. 
In the following, we design a differentiable and smooth approximation of the robust value function. Specifically,  
consider a smoothed robust Bellman operator $\mathbf T_{\sigma}^{\pi}$ using the LSE function \cite{wang2021online,wang2022policy}:
\begin{align}\label{eq:smoothed bellman}
    \mathbf T_{\sigma}^{\pi} V(s)&=\mathbb{E}_{A\sim\pi(\cdot|s)}\bigg[r(s,A)+\gamma(1-\delta)\sum_{{s'}\in\mcs} p^{A}_{s,{s'}} V({s'})+\gamma \delta \lse(\sigma, V)\bigg],
\end{align}
where
$\lse(\sigma,V)=\frac{\log (\sum^d_{i=1} e^{\sigma V(i)})}{\sigma}$ for $V\in\mathbb{R}^d$ and some $\sigma<0$.
The approximation error $|\lse(\sigma,V)-\min V|\to 0$ as $\sigma\to-\infty$, and hence the fixed point  of $\mathbf T_{\sigma}^{\pi}$, denoted by $V^\pi_\sigma$, is an approximation of the robust value function $V^\pi$ \cite{wang2022policy}. We refer to $V^\pi_\sigma$ as the smoothed robust value function and define the smoothed robust action-value function as $Q^\pi_\sigma(s,a)\triangleq r(s,a)+\gamma(1-\delta)\sum_{{s'}\in\mcs} p^{a}_{s,{s'}} V^\pi_\sigma({s'}) +\gamma \delta\lse(\sigma, V^\pi_\sigma)$. 
It can be shown that for any $\pi$, as $\sigma\rightarrow-\infty$, $\|V_r^\pi-V_{\sigma,r}^\pi \|\to 0$ and $\|V_c^\pi-V_{\sigma,c}^\pi \|\to 0$ \cite{wang2021online}.

The gradient of $V^{\pi_\theta}_\sigma$ can be computed explicitly \cite{wang2022policy}:
$
    \nabla V^{\pit}_\sigma(s)=B(s,\theta)+\frac{\gamma \delta\sum_{s\in\mcs} e^{\sigma \sv^{\pit}(s)} B(s,\theta)}{(1-\gamma)\sum_{s\in\mcs} e^{\sigma \sv^{\pit}(s)}},
$
where $B(s,\theta)\triangleq\frac{1}{1-\gamma+\gamma \delta} \sum_{{s'}\in\mcs} d^{\pit}_{s,\mathsf P}({s'})\sum_{a\in\mca} \nabla \pit(a|{s'})\sq^{\pit}({s'},a)$, and $d^{\pit}_{s,\mathsf P}(\cdot)$ is the visitation distribution of $\pi_\theta$ under $\mathsf P$ starting from $s$.  Denote the smoothed Lagrangian function by $V_\sigma^L(\theta,\lambda)\triangleq V^{\pit}_{\sigma,r}(\rho)+\lambda(V_{\sigma,c}^{\pit}(\rho)-b)$. 
The following lemma shows that $\nabla V^L_\sigma$ is Lipschitz.
\begin{lemma}\label{lemma:lip}
$\nabla V^L_\sigma$ is Lipschitz in $\theta$ and $\lambda$. And hence Assumption \ref{assumption:lipschitz} holds for $V_\sigma^L$.
\end{lemma}


A natural idea is to use the smoothed robust value functions to replace the ones in \eqref{eq:dual}:
\begin{align}\label{eq:smoothed dual}
    \min_{\lambda\geq 0} \max_{\pi\in\Pi_\Theta} V^\pi_{\sigma,r}(\rho)+\lambda(V_{\sigma,c}^\pi(\rho)-b).
\end{align}
As will be shown below in Lemma \ref{lemma2}, this approximation can be arbitrarily close to the original problem in \eqref{eq:dual} as $\sigma\rightarrow-\infty$. 
We first show that under Assumption \ref{ass:slater}, the following Slater's condition holds for the smoothed problem in \eqref{eq:smoothed dual}.
\begin{lemma}
Let $\sigma$ be sufficiently small such that $\|V^\pi_{\sigma,c}-V_c^\pi\|<\zeta$ for any $\pi$, then there exists $\zeta'>0$ and a policy $\pi'\in\Pi_\Theta$ s.t. $V_{\sigma,c}^{\pi'}(\rho)-b\geq \zeta'$.
\end{lemma}

The following lemma shows that the optimal dual variable for \eqref{eq:smoothed dual} is also bounded.
\begin{lemma}\label{lemma:222}
Denote the optimal solution of \eqref{eq:smoothed dual} by $(\lambda^*,\pi_{\theta^*})$. Then $\lambda^*\in\left[0,\frac{2C_\sigma}{{\zeta'}}\right]$, where $C_\sigma$ is the upper bound of smoothed robust value functions $V^\pi_{\sigma,c}$.
\end{lemma}
Denote by $\Lambda^*=\max\left\{\frac{2C_\sigma}{\zeta'}, \frac{2}{\zeta(1-\gamma)}\right\}$, then  problems \eqref{eq:bounded dual} and \eqref{eq:smoothed dual} are equivalent to the following bounded ones:
$
     \min_{\lambda\in\left[0,\Lambda^*\right]} \max_{\pi\in\Pi_\Theta} V^\pi_r(\rho)+\lambda(V_c^\pi(\rho)-b),
$
and 
\begin{align}\label{eq:boundedsmoothed}
     \min_{\lambda\in[0,\Lambda^*]} \max_{\pi\in\Pi_\Theta} V^\pi_{\sigma,r}(\rho)+\lambda(V_{\sigma,c}^\pi(\rho)-b).
\end{align}
The following lemma shows that the two problems are within a gap of $\mathcal{O}(\epsilon)$.
\begin{lemma}\label{lemma2}
Choose a small enough $\sigma$ such that $\|V_r^\pi-V_{\sigma,r}^\pi \|\leq \epsilon$ and $\|V_c^\pi-V_{\sigma,c}^\pi \|\leq \epsilon$. Then
\begin{align*}
    \left|\min_{\lambda\in[0,\Lambda^*]} \max_{\pi\in\Pi_\Theta} V^\pi_{\sigma,r}(\rho)+\lambda(V_{\sigma,c}^\pi(\rho)-b)-\min_{\lambda\in[0,\Lambda^*]} \max_{\pi\in\Pi_\Theta} V^\pi_{r}(\rho)+\lambda(V_c^\pi(\rho)-b)\right|  \leq \left(1+\Lambda^*\right)\epsilon.
\end{align*}
\end{lemma}

In the following, we hence focus on the smoothed dual problem in \eqref{eq:boundedsmoothed}, which is an accurate approximation of the original problem \eqref{eq:bounded dual}. 
Denote  the gradient mapping of the smoothed Lagrangian function $V_\sigma^L$ by
\begin{align}
G_t\triangleq\left[
\begin{array}{lr}
     &\beta_t \left(\lambda_t-\mathbf{\prod}_{[0,\Lambda^*]} \left(\lambda_t -\frac{1}{\beta_t} \left(\nabla_\lambda V^L_\sigma (\theta_t,\lambda_t) \right)\right)\right)  \\
     & \alpha_t\left(\theta_t-\mathbf{\prod}_\Theta \left(\theta_t+\frac{1}{\alpha_t} \left(\nabla_\theta V^L_\sigma (\theta_t,\lambda_t) \right) \right)\right) 
\end{array}
\right].
\end{align}
Applying our RPD algorithm in \eqref{eq:boundedsmoothed}, we have the following convergence guarantee.
\begin{corollary}
If we 
 set step sizes $\alpha_t,\beta_t,$ and $b_t$ as in Section \ref{sec:constants} and set $T=\mathcal{O}(\epsilon^{-4})$, then 
$
    \min_{1\leq t\leq T}\|G_t\|\leq 2\epsilon.
$
\end{corollary}
This corollary implies that our robust primal-dual algorithm  converges to a stationary point of the min-max problem \eqref{eq:boundedsmoothed} under the $\delta$-contamination model, with a complexity of $\mathcal{O}(\epsilon^{-4})$.

\begin{wrapfigure}{L}{0.5\textwidth}
\begin{minipage}{0.5\textwidth}
\begin{algorithm}[H]
        \caption{Smoothed Robust TD \cite{wang2022policy}}
\label{alg:rtd}
\textbf{Input}:   $T_{\text{inner}}$, $\pi$, $\sigma$, $c$\\
\textbf{Initialization}: $Q_0$, $s_0$
\begin{algorithmic} 
\FOR {$t=0,1,...,T_{\text{inner}}-1$}
\STATE \hspace{-0.2cm}Choose $a_t\sim\pi(\cdot|s_t)$ and observe $c_t, s_{t+1}$
\STATE \hspace{-0.2cm}$V_t(s)\leftarrow \sum_{a\in\mca} \pi(a|s)Q_t(s,a)$ \text{ for all }$s\in\mcs$
\STATE \hspace{-0.2cm}{$Q_{t+1}(s_t,a_t)\leftarrow Q_t(s_t,a_t)+\alpha_t \big(c_t+\gamma (1-\quad\delta) \cdot V_t(s_{t+1})+\gamma \delta \cdot \text{LSE}(\sigma, V_t)-Q_t(s_t,a_t)\big)$}
\ENDFOR
\end{algorithmic}
\textbf{Output}: $Q_{T_{\text{inner}},c}\triangleq Q_{T_{\text{inner}}}$
\end{algorithm}
\end{minipage}
\end{wrapfigure}

Note that Algorithm \ref{alg:rpd} assumes knowledge of the smoothed robust value functions which may not be available in practice. Different from the non-robust value function which can be estimated using Monte Carlo, robust value functions are the value function corresponding to the worst-case transition kernel from which no samples are directly taken. To solve this issue, we adopt the smoothed robust TD algorithm (Algorithm \ref{alg:rtd}) from \cite{wang2022policy} to estimate the smoothed robust value functions.

  It was shown that the smoothed robust TD algorithm converges to the smoothed robust value function with a sample complexity of $\mathcal{O}(\epsilon^{-2})$ \cite{wang2022policy} under the tabular case. We then construct our online and model-free  RPD algorithm as in Algorithm \ref{alg:srpd}. We note that Algorithm \ref{alg:srpd} is for the tabular setting with finite $\mcs$ and $\mca$. It can be easily extended to the case with large/continuous $\mcs$ and $\mca$ using function approximation.

      \begin{algorithm}[H]
\caption{Online Robust Primal-Dual algorithm}
\label{alg:srpd}
\textbf{Input}:   $T$, $\sigma$, $\epe$, $\beta_t$, $\alpha_t$, $b_t$,$r$,$c$\\
\textbf{Initialization}: $\lambda_0$, $\theta_0$\
\begin{algorithmic}
\FOR {$t=0,1,...,T-1$}
\STATE {Set $T_{\text{inner}}=\mathcal{O}\left(\frac{(t+1)^{1.5}}{\epe^2}\right)$ and run Algorithm \ref{alg:rtd} for $r$ and $c$,   output  $Q_{T_{\text{inner}},r}$, $Q_{T_{\text{inner}},c}$}
\STATE {$\hat{V}^{\pi_{\theta_t}}_{\sigma,r}(s)\leftarrow \sum_a\pi_{\theta_t}(a|s)Q_{T_{\text{inner}},r}(s,a) $, $\hat{V}^{\pi_{\theta_t}}_{\sigma,c}(s) \leftarrow \sum_a\pi_{\theta_t}(a|s)Q_{T_{\text{inner}},c}(s,a) $}
\STATE{$\hat{V}^{\pi_{\theta_t}}_{\sigma,r}(\rho) \leftarrow \sum_s \rho(s)\hat{V}^{\pi_{\theta_t}}_{\sigma,r}(s) $, $\hat{V}^{\pi_{\theta_t}}_{\sigma,c}(\rho) \leftarrow \sum_s \rho(s)\hat{V}^{\pi_{\theta_t}}_{\sigma,c}(s) $}
\STATE {$\lambda_{t+1}\leftarrow \mathbf{\prod}_{[0,\Lambda^*]} \left(\lambda_t -\frac{1}{\beta_t} \left(\hat{V}^{\pi_{\theta_t}}_{\sigma,c}(\rho)-b \right)-\frac{b_t}{\beta_t}\lambda_t  \right)$}
\STATE {$\theta_{t+1}\leftarrow \mathbf{\prod}_\Theta \left(\theta_t+\frac{1}{\alpha_t} \left(\nabla_{\theta}\hat{V}^{\pi_{\theta_t}}_{\sigma,r}(\rho)+\lambda_{t+1}\nabla_\theta \hat{V}^{\pi_{\theta_t}}_{\sigma,c}(\rho)  \right) \right)$}
\ENDFOR
\end{algorithmic}
\textbf{Output}: $\theta_T$
\end{algorithm}

Algorithm \ref{alg:srpd} can be viewed as a biased stochastic gradient descent-ascent algorithm. It is a sample-based algorithm without assuming any knowledge of robust value functions, and can be performed in an online fashion. 
We further extend the convergence results in Theorem \ref{thm:1} to the model-free setting, and characterize the following finite-time error bound of Algorithm \ref{alg:srpd}. Similarly, Algorithm \ref{alg:srpd} can be shown to achieve a $2\epsilon$-feasible policy almost surely. 

{Under the online model-free setting, the estimation of the robust value functions is biased. Therefore, the analysis is more challenging than the existing literature, where it is usually assumed that the gradients are exact. We develop a new method to bound the bias accumulated in every iteration of the algorithm, and establish the final convergence results.}
\begin{theorem}\label{thm2}
Consider the same conditions as in Theorem \ref{thm:1}. Let $\epe=\mathcal{O}(\epsilon^2)$ and $T=\mathcal{O}(\epsilon^{-4})$, then $
    \min_{1\leq t\leq T}\|G_t\|\leq (1+\sqrt{2})\epsilon.
$
\end{theorem}

\section{Numerical Results}

In this section, we numerically demonstrate the robustness of our algorithm in terms of both maximizing robust reward value function and satisfying constraints under model uncertainty. We compare our RPD algorithm with the heuristic algorithms in \cite{russel2021lyapunov,mankowitz2020robust} and the vanilla non-robust primal-dual method. Based on the idea of "robust policy evaluation" + "non-robust policy improvement" in \cite{russel2021lyapunov,mankowitz2020robust}, we combine the robust TD algorithm \ref{alg:rtd} with non-robust vanilla policy gradient method \cite{sutton1999policy}, which we refer to as the heuristic primal-dual algorithm.
%
Several environments, including   Garnet  \cite{archibald1995generation},     $8\times 8$ Frozen-Lake and Taxi environments from OpenAI \cite{brockman2016openai}, are investigated. 

We first run the  algorithm and store the obtained policies $\pi_t$ at each time step. At each time step,  we run robust TD with a sample size 200 for 30 times to estimate the  objective  $V_r(\rho)$ and the constraint $V_c(\rho)$. We then plot them v.s. the number of iterations $t$.  The upper and lower envelopes of the curves correspond to the 95 and 5 percentiles of the 30 curves, respectively. We repeat the experiment for two different values of $\delta=0.2,0.3$.

\textbf{Garnet problem.}
A Garnet problem can be specified by $\mathcal{G}(S_n,A_n)$, where the state space $\mcs$ has $S_n$ states $(s_1,...,s_{S_n})$ and action space has $A_n$ actions $(a_1,...,a_{A_n})$. The agent can take any actions in any state, and receives a randomly generated reward/utility signal {generated  from the uniform distribution on [0,1]}. The transition kernels are also randomly generated. The comparison results are shown in Fig.\ref{Fig.g}.

\textbf{$8\times 8$ Frozen-Lake problem.}
We then compare the three algorithms under the $8\times 8$ Frozen-lake problem setting in Fig.\ref{Fig.8f}.
The Frozen-Lake problem involves a frozen lake of size $8\times 8$ which contains several "holes". The agent aims to cross the lake from the start point to the end point without falling into any holes. The agent receives $r=-10$ {and $c=0$} when falling in a hole, receives $r=20$ {and $c=1$} when arrive at the end point; At other times, the agent receives $r=0$  {and a randomly generated utility $c$ according to the uniform distribution on [0,1]}.

\textbf{Taxi problem.}
We then compare the three algorithms under the Taxi problem environment. 
The taxi problem simulates a taxi driver in a $5\times 5$ map. There are four designated locations in the grid world and a passenger occurs at a random location of the designated four locations at the start of each episode. The goal of the driver is to first pick up the passenger and then to drop off at another specific location. The driver receives $r=20$ for each successful drop-off, and always receives $r=-1$ at other times. 
{We randomly generate the utility according to the uniform distribution on [0,1] for each state-action pair.}  The results are shown in Fig.\ref{Fig.t}.

\begin{figure}[ht]
\begin{center}
\subfigure[$V_c$ when $\delta=0.2$.]{
\begin{minipage}[t]{0.23\linewidth}
\centering
\label{Fig.g1c}
\includegraphics[width=1.47 in]{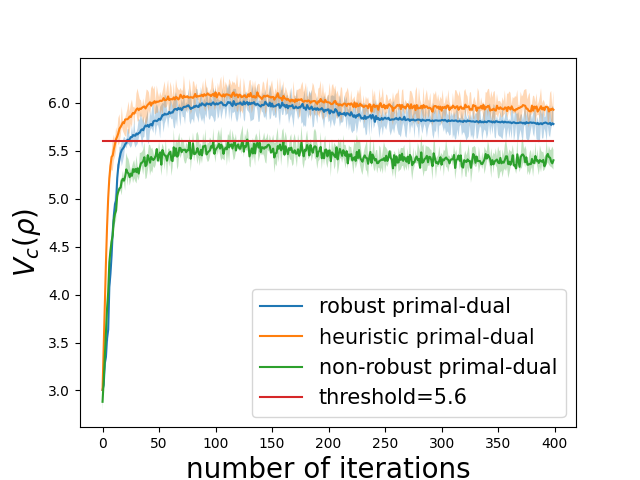}
\end{minipage}}
\subfigure[$V_r$ when $\delta=0.2$.]{
\begin{minipage}[t]{0.23\linewidth}
\centering
\label{Fig.g1r}
\includegraphics[width=1.47 in]{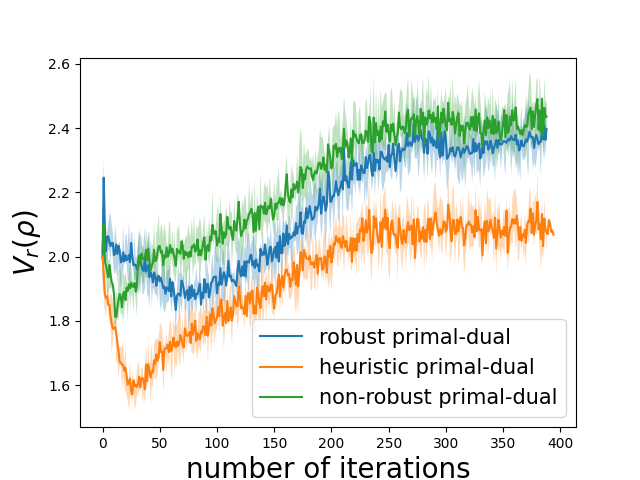}
\end{minipage}}
\subfigure[$V_c$ when $\delta=0.3$.]{
\begin{minipage}[t]{0.23\linewidth}
\centering
\label{Fig.g2c}
\includegraphics[width=1.47 in]{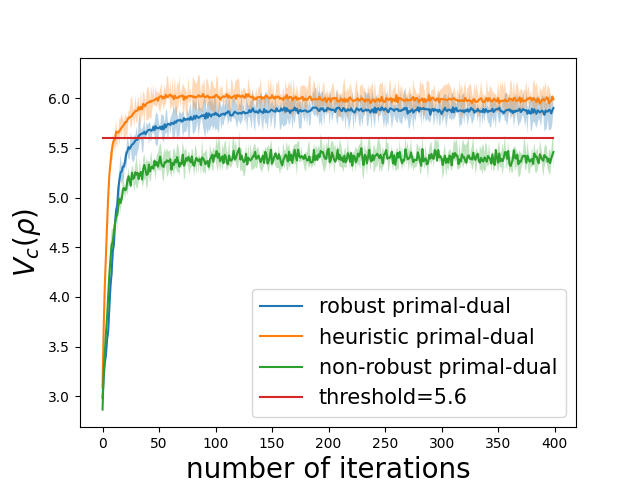}
\end{minipage}}
\subfigure[$V_r$ when $\delta=0.3$.]{
\begin{minipage}[t]{0.23\linewidth}
\centering
\label{Fig.g2r}
\includegraphics[width=1.47 in]{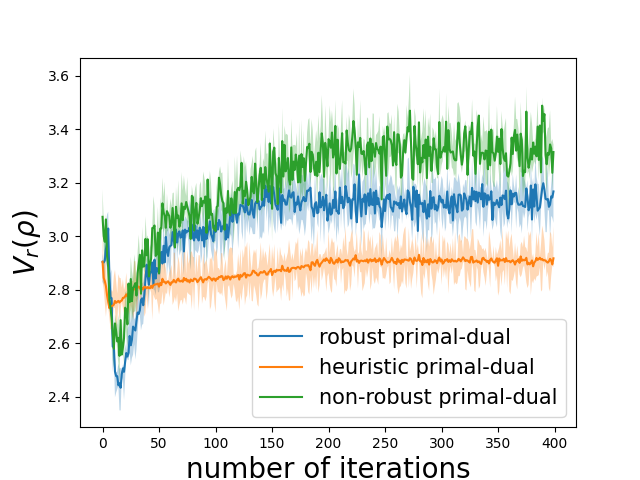}
\end{minipage}}
\caption{Comparison on Garnet Problem $\mathcal{G}(20,10)$.}
\label{Fig.g}
\end{center}
\vskip -0.2in
\end{figure}

\begin{figure}[ht]
\begin{center}
\subfigure[$V_c$ when $\delta=0.2$.]{
\begin{minipage}[t]{0.23\linewidth}
\centering
\label{Fig.8f1c}
\includegraphics[width=1.47 in]{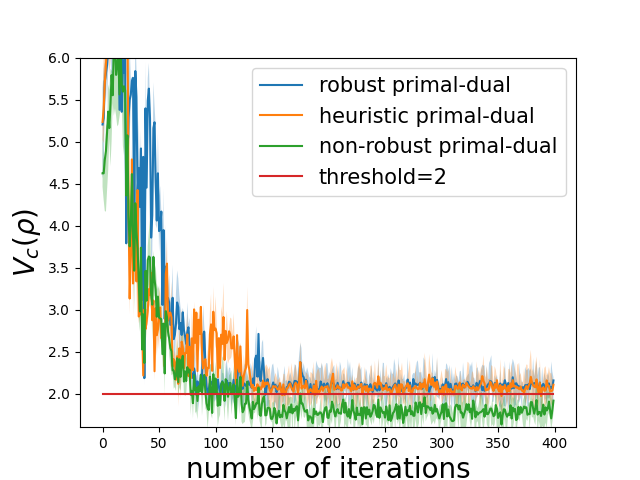}
\end{minipage}}
\subfigure[$V_r$ when $\delta=0.2$.]{
\begin{minipage}[t]{0.23\linewidth}
\centering
\label{Fig.8f1r}
\includegraphics[width=1.47 in]{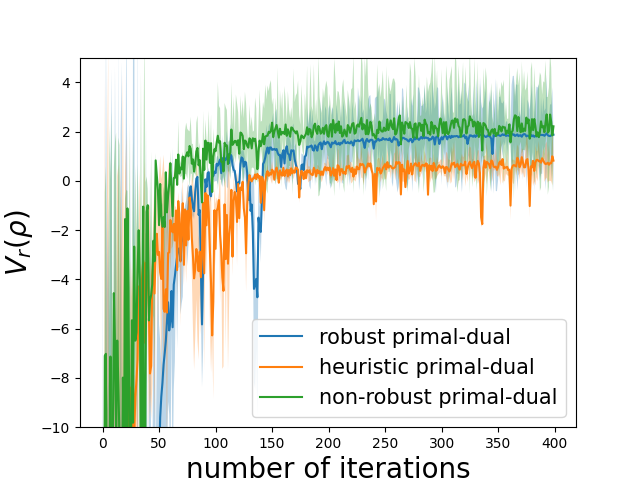}
\end{minipage}}
\subfigure[$V_c$ when $\delta=0.3$.]{
\begin{minipage}[t]{0.23\linewidth}
\centering
\label{Fig.8f2c}
\includegraphics[width=1.47 in]{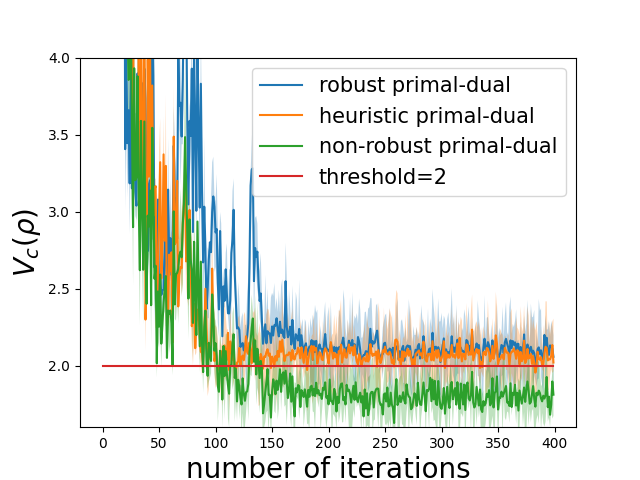}
\end{minipage}}
\subfigure[$V_r$ when $\delta=0.3$.]{
\begin{minipage}[t]{0.23\linewidth}
\centering
\label{Fig.8f2r}
\includegraphics[width=1.47 in]{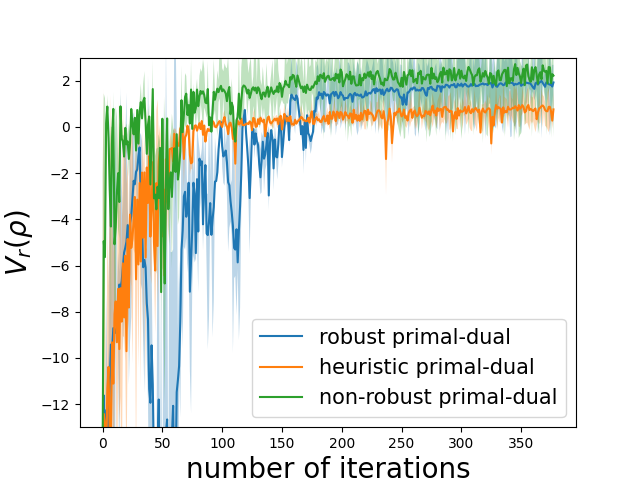}
\end{minipage}}
\caption{Comparison on $8\times 8$ Frozen-Lake Problem.}
\label{Fig.8f}
\end{center}
\vskip -0.2in
\end{figure}
\begin{figure}[htb]
\begin{center}
\subfigure[$V_c$ when $\delta=0.2$.]{
\begin{minipage}[t]{0.23\linewidth}
\centering
\label{Fig.t1c}
\includegraphics[width=1.47 in]{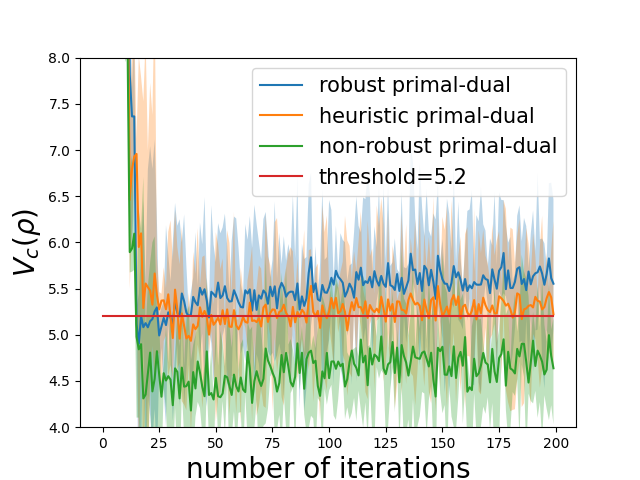}
\end{minipage}}
\subfigure[$V_r$ when $\delta=0.2$.]{
\begin{minipage}[t]{0.23\linewidth}
\centering
\label{Fig.t1r}
\includegraphics[width=1.47 in]{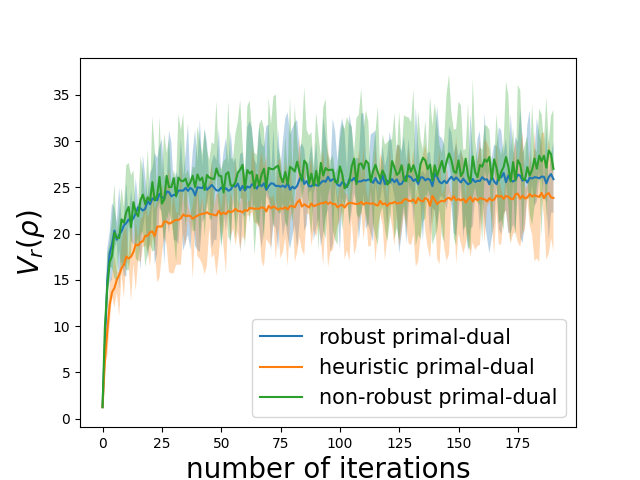}
\end{minipage}}
\subfigure[$V_c$ when $\delta=0.3$.]{
\begin{minipage}[t]{0.23\linewidth}
\centering
\label{Fig.t2c}
\includegraphics[width=1.47 in]{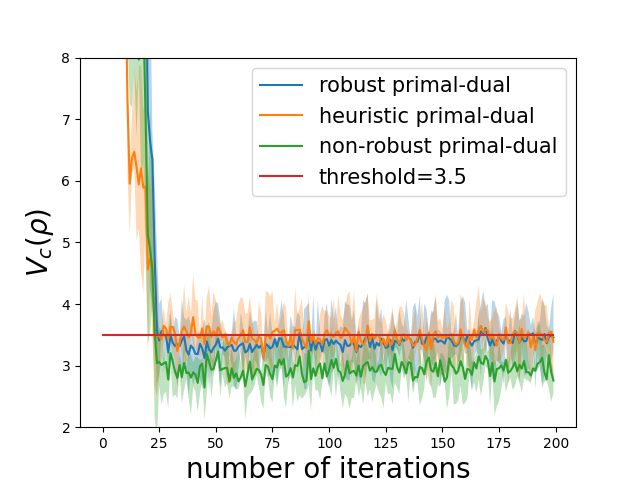}
\end{minipage}}
\subfigure[$V_r$ when $\delta=0.3$.]{
\begin{minipage}[t]{0.23\linewidth}
\centering
\label{Fig.t2r}
\includegraphics[width=1.47 in]{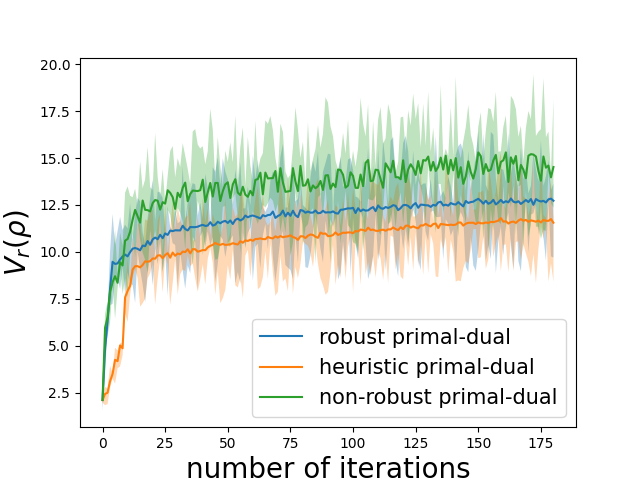}
\end{minipage}}
\caption{Comparison on Taxi Problem.}
\label{Fig.t}
\end{center}
\vskip -0.2in
\end{figure}

From the experiment results above, it can be seen that:
(1) Both our RPD algorithm and the heuristic primal-dual approach find feasible policies satisfying the constraint under the worst-case scenario, i.e., $V^\pi_c\geq b$. However, the non-robust primal-dual method fails to find a feasible solution that satisfy the constraint under the worst-case scenario. 
(2) Compared to the heuristic PD method, our RPD method can obtain more reward and can find a more robust policy while satisfying the robust constraint. Note that the non-robust PD method obtain more reward, but this is because the policy it finds  violates the robust constraint. 
Our experiments demonstrate that among the three algorithms, our RPD algorithm is the best one which optimizes the worst-case reward performance while satisfying the robust constraints on the utility.


\color{black}

\section{Conclusion}
In this paper, we formulate the problem of robust constrained reinforcement learning under model uncertainty, where the goal is to guarantee that constraints are satisfied for all MDPs in the uncertainty set, and to maximize the worst-case reward performance over the uncertainty set. We propose a robust primal-dual algorithm, and theoretically characterize its convergence, complexity and robust feasibility. Our algorithm guarantees convergence to a feasible solution, and outperforms the other two heuristic algorithms. We further investigate a concrete example with $\delta$-contamination uncertainty set, and construct online and model-free robust primal-dual algorithm. {Our methodology can also be readily extended to problems with other uncertainty sets like KL-divergence, total variation and Wasserstein distance. 
The major challenge lies in deriving the robust policy gradient, and further designing model-free algorithm to estimate the robust value function. We also expect that Assumption \ref{assumption:lipschitz} and further results can be derived if some proper smoothing technique is used.}

\textbf{Limitations: }It is of future interest to generalize our results to other types of uncertainty sets, e.g., ones defined by KL divergence, total variation, Wasserstein distance. \textbf{Negative societal impact:} This work is a theoretical study. To the best of the authors’ knowledge, it does not have any potential negative impact on the society.

\newpage

\bibliographystyle{unsrtnat}
\bibliography{references}  

\begin{thebibliography}{81}
\providecommand{\natexlab}[1]{#1}
\providecommand{\url}[1]{\texttt{#1}}
\expandafter\ifx\csname urlstyle\endcsname\relax
  \providecommand{\doi}[1]{doi: #1}\else
  \providecommand{\doi}{doi: \begingroup \urlstyle{rm}\Url}\fi

\bibitem[Nilim and El~Ghaoui(2004)]{nilim2004robustness}
Arnab Nilim and Laurent El~Ghaoui.
\newblock Robustness in {{M}arkov} decision problems with uncertain transition
  matrices.
\newblock In \emph{Proc. Advances in Neural Information Processing Systems
  (NIPS)}, pages 839--846, 2004.

\bibitem[Iyengar(2005)]{iyengar2005robust}
Garud~N Iyengar.
\newblock Robust dynamic programming.
\newblock \emph{Mathematics of Operations Research}, 30\penalty0 (2):\penalty0
  257--280, 2005.

\bibitem[Bagnell et~al.(2001)Bagnell, Ng, and Schneider]{bagnell2001solving}
J~Andrew Bagnell, Andrew~Y Ng, and Jeff~G Schneider.
\newblock Solving uncertain {M}arkov decision processes.
\newblock 09 2001.

\bibitem[Altman(1999)]{altman1999constrained}
Eitan Altman.
\newblock \emph{Constrained {M}arkov decision processes: stochastic modeling}.
\newblock Routledge, 1999.

\bibitem[Russel et~al.(2020)Russel, Benosman, and Van~Baar]{russel2020robust}
Reazul~Hasan Russel, Mouhacine Benosman, and Jeroen Van~Baar.
\newblock Robust constrained-{MDP}s: Soft-constrained robust policy
  optimization under model uncertainty.
\newblock \emph{arXiv preprint arXiv:2010.04870}, 2020.

\bibitem[Mankowitz et~al.(2020)Mankowitz, Calian, Jeong, Paduraru, Heess,
  Dathathri, Riedmiller, and Mann]{mankowitz2020robust}
Daniel~J Mankowitz, Dan~A Calian, Rae Jeong, Cosmin Paduraru, Nicolas Heess,
  Sumanth Dathathri, Martin Riedmiller, and Timothy Mann.
\newblock Robust constrained reinforcement learning for continuous control with
  model misspecification.
\newblock \emph{arXiv preprint arXiv:2010.10644}, 2020.

\bibitem[Sutton et~al.(1999)Sutton, McAllester, Singh, Mansour,
  et~al.]{sutton1999policy}
Richard~S Sutton, David~A McAllester, Satinder~P Singh, Yishay Mansour, et~al.
\newblock Policy gradient methods for reinforcement learning with function
  approximation.
\newblock In \emph{Proc. Advances in Neural Information Processing Systems
  (NIPS)}, volume~99, pages 1057--1063. Citeseer, 1999.

\bibitem[Huber(1965)]{hub65}
P.~J. Huber.
\newblock A robust version of the probability ratio test.
\newblock \emph{Ann. Math. Statist.}, 36:\penalty0 1753--1758, 1965.

\bibitem[Du et~al.(2018)Du, Wang, Balakrishnan, Ravikumar, and
  Singh]{du2018robust}
Simon~S Du, Yining Wang, Sivaraman Balakrishnan, Pradeep Ravikumar, and Aarti
  Singh.
\newblock Robust nonparametric regression under {H}uber's
  $\epsilon$-contamination model.
\newblock \emph{arXiv preprint arXiv:1805.10406}, 2018.

\bibitem[Huber and Ronchetti(2009)]{huber2009robust}
PJ~Huber and EM~Ronchetti.
\newblock \emph{Robust Statistics}.
\newblock John Wiley \& Sons, Inc, 2009.

\bibitem[Nishimura and Ozaki(2004)]{nishimura2004search}
Kiyohiko~G Nishimura and Hiroyuki Ozaki.
\newblock Search and knightian uncertainty.
\newblock \emph{Journal of Economic Theory}, 119\penalty0 (2):\penalty0
  299--333, 2004.

\bibitem[Nishimura and Ozaki(2006)]{Kiyohiko2006}
Kiyohiko~G. Nishimura and Hiroyuki Ozaki.
\newblock An axiomatic approach to $\epsilon$-contamination.
\newblock \emph{Economic Theory}, 27\penalty0 (2):\penalty0 333--340, 2006.

\bibitem[Prasad et~al.(2020{\natexlab{a}})Prasad, Srinivasan, Balakrishnan, and
  Ravikumar]{prasad2020learning}
Adarsh Prasad, Vishwak Srinivasan, Sivaraman Balakrishnan, and Pradeep
  Ravikumar.
\newblock On learning ising models under {H}uber's contamination model.
\newblock \emph{Proc. Advances in Neural Information Processing Systems
  (NeurIPS)}, 33, 2020{\natexlab{a}}.

\bibitem[Prasad et~al.(2020{\natexlab{b}})Prasad, Suggala, Balakrishnan, and
  Ravikumar]{prasad2020robust}
Adarsh Prasad, Arun~Sai Suggala, Sivaraman Balakrishnan, and Pradeep Ravikumar.
\newblock Robust estimation via robust gradient estimation.
\newblock \emph{Journal of the Royal Statistical Society: Series B (Statistical
  Methodology)}, 82\penalty0 (3):\penalty0 601--627, 2020{\natexlab{b}}.

\bibitem[Wang and Zou(2021)]{wang2021online}
Yue Wang and Shaofeng Zou.
\newblock Online robust reinforcement learning with model uncertainty.
\newblock In \emph{Proc. Advances in Neural Information Processing Systems
  (NeurIPS)}, 2021.

\bibitem[Wang and Zou(2022)]{wang2022policy}
Yue Wang and Shaofeng Zou.
\newblock Policy gradient method for robust reinforcement learning.
\newblock In \emph{Proc. International Conference on Machine Learning (ICML)},
  2022.

\bibitem[Ding et~al.(2020)Ding, Zhang, Basar, and Jovanovic]{ding2020natural}
Dongsheng Ding, Kaiqing Zhang, Tamer Basar, and Mihailo Jovanovic.
\newblock Natural policy gradient primal-dual method for constrained {M}arkov
  decision processes.
\newblock In \emph{Proc. Advances in Neural Information Processing Systems
  (NeurIPS)}, volume~33, pages 8378--8390, 2020.

\bibitem[Ding et~al.(2021)Ding, Wei, Yang, Wang, and
  Jovanovic]{ding2021provably}
Dongsheng Ding, Xiaohan Wei, Zhuoran Yang, Zhaoran Wang, and Mihailo Jovanovic.
\newblock Provably efficient safe exploration via primal-dual policy
  optimization.
\newblock In \emph{Proc. International Conference on Artifical Intelligence and
  Statistics (AISTATS)}, pages 3304--3312. PMLR, 2021.

\bibitem[Li et~al.(2021{\natexlab{a}})Li, Guan, Zou, Xu, Liang, and
  Lan]{li2021faster}
Tianjiao Li, Ziwei Guan, Shaofeng Zou, Tengyu Xu, Yingbin Liang, and Guanghui
  Lan.
\newblock Faster algorithm and sharper analysis for constrained {M}arkov
  decision process.
\newblock \emph{arXiv preprint arXiv:2110.10351}, 2021{\natexlab{a}}.

\bibitem[Liu et~al.(2021)Liu, Zhou, Kalathil, Kumar, and Tian]{liu2021fast}
Tao Liu, Ruida Zhou, Dileep Kalathil, PR~Kumar, and Chao Tian.
\newblock Fast global convergence of policy optimization for constrained
  {MDP}s.
\newblock \emph{arXiv preprint arXiv:2111.00552}, 2021.

\bibitem[Ying et~al.(2021)Ying, Ding, and Lavaei]{ying2021dual}
Donghao Ying, Yuhao Ding, and Javad Lavaei.
\newblock A dual approach to constrained {M}arkov decision processes with
  entropy regularization.
\newblock \emph{arXiv preprint arXiv:2110.08923}, 2021.

\bibitem[Paternain et~al.(2019)Paternain, Chamon, Calvo-Fullana, and
  Ribeiro]{paternain2019constrained}
Santiago Paternain, Luiz Chamon, Miguel Calvo-Fullana, and Alejandro Ribeiro.
\newblock Constrained reinforcement learning has zero duality gap.
\newblock In \emph{Proc. Advances in Neural Information Processing Systems
  (NeurIPS)}, volume~32, 2019.

\bibitem[Paternain et~al.(2022)Paternain, Calvo-Fullana, Chamon, and
  Ribeiro]{paternain2022safe}
Santiago Paternain, Miguel Calvo-Fullana, Luiz~FO Chamon, and Alejandro
  Ribeiro.
\newblock Safe policies for reinforcement learning via primal-dual methods.
\newblock \emph{IEEE Transactions on Automatic Control}, 2022.

\bibitem[Liang et~al.(2018)Liang, Que, and Modiano]{liang2018accelerated}
Qingkai Liang, Fanyu Que, and Eytan Modiano.
\newblock Accelerated primal-dual policy optimization for safe reinforcement
  learning.
\newblock \emph{arXiv preprint arXiv:1802.06480}, 2018.

\bibitem[Stooke et~al.(2020)Stooke, Achiam, and Abbeel]{stooke2020responsive}
Adam Stooke, Joshua Achiam, and Pieter Abbeel.
\newblock Responsive safety in reinforcement learning by pid lagrangian
  methods.
\newblock In \emph{Proc. International Conference on Machine Learning (ICML)},
  pages 9133--9143. PMLR, 2020.

\bibitem[Tessler et~al.(2018)Tessler, Mankowitz, and Mannor]{tessler2018reward}
Chen Tessler, Daniel~J Mankowitz, and Shie Mannor.
\newblock Reward constrained policy optimization.
\newblock \emph{arXiv preprint arXiv:1805.11074}, 2018.

\bibitem[Yu et~al.(2019)Yu, Yang, Kolar, and Wang]{yu2019convergent}
Ming Yu, Zhuoran Yang, Mladen Kolar, and Zhaoran Wang.
\newblock Convergent policy optimization for safe reinforcement learning.
\newblock In \emph{Proc. Advances in Neural Information Processing Systems
  (NeurIPS)}, volume~32, 2019.

\bibitem[Zheng and Ratliff(2020)]{zheng2020constrained}
Liyuan Zheng and Lillian Ratliff.
\newblock Constrained upper confidence reinforcement learning.
\newblock In \emph{Learning for Dynamics and Control}, pages 620--629. PMLR,
  2020.

\bibitem[Efroni et~al.(2020)Efroni, Mannor, and Pirotta]{efroni2020exploration}
Yonathan Efroni, Shie Mannor, and Matteo Pirotta.
\newblock Exploration-exploitation in constrained {MDP}s.
\newblock \emph{arXiv preprint arXiv:2003.02189}, 2020.

\bibitem[Auer et~al.(2008)Auer, Jaksch, and Ortner]{auer2008near}
Peter Auer, Thomas Jaksch, and Ronald Ortner.
\newblock Near-optimal regret bounds for reinforcement learning.
\newblock In \emph{Proc. Advances in Neural Information Processing Systems
  (NIPS)}, volume~21, 2008.

\bibitem[Achiam et~al.(2017)Achiam, Held, Tamar, and
  Abbeel]{achiam2017constrained}
Joshua Achiam, David Held, Aviv Tamar, and Pieter Abbeel.
\newblock Constrained policy optimization.
\newblock In \emph{Proc. International Conference on Machine Learning (ICML)},
  pages 22--31. PMLR, 2017.

\bibitem[Liu et~al.(2020)Liu, Ding, and Liu]{liu2020ipo}
Yongshuai Liu, Jiaxin Ding, and Xin Liu.
\newblock Ipo: Interior-point policy optimization under constraints.
\newblock In \emph{Proc. Conference on Artificial Intelligence (AAAI)},
  volume~34, pages 4940--4947, 2020.

\bibitem[Chow et~al.(2018)Chow, Nachum, Duenez-Guzman, and
  Ghavamzadeh]{chow2018lyapunov}
Yinlam Chow, Ofir Nachum, Edgar Duenez-Guzman, and Mohammad Ghavamzadeh.
\newblock A lyapunov-based approach to safe reinforcement learning.
\newblock In \emph{Proc. Advances in Neural Information Processing Systems
  (NeurIPS)}, volume~31, 2018.

\bibitem[Dalal et~al.(2018)Dalal, Dvijotham, Vecerik, Hester, Paduraru, and
  Tassa]{dalal2018safe}
Gal Dalal, Krishnamurthy Dvijotham, Matej Vecerik, Todd Hester, Cosmin
  Paduraru, and Yuval Tassa.
\newblock Safe exploration in continuous action spaces.
\newblock \emph{arXiv preprint arXiv:1801.08757}, 2018.

\bibitem[Xu et~al.(2021)Xu, Liang, and Lan]{xu2021crpo}
Tengyu Xu, Yingbin Liang, and Guanghui Lan.
\newblock Crpo: A new approach for safe reinforcement learning with convergence
  guarantee.
\newblock In \emph{Proc. International Conference on Machine Learning (ICML)},
  pages 11480--11491. PMLR, 2021.

\bibitem[Yang et~al.(2020)Yang, Rosca, Narasimhan, and
  Ramadge]{yang2020projection}
Tsung-Yen Yang, Justinian Rosca, Karthik Narasimhan, and Peter~J Ramadge.
\newblock Projection-based constrained policy optimization.
\newblock \emph{arXiv preprint arXiv:2010.03152}, 2020.

\bibitem[Satia and Lave~Jr(1973)]{satia1973markovian}
Jay~K Satia and Roy~E Lave~Jr.
\newblock {M}arkovian decision processes with uncertain transition
  probabilities.
\newblock \emph{Operations Research}, 21\penalty0 (3):\penalty0 728--740, 1973.

\bibitem[Wiesemann et~al.(2013)Wiesemann, Kuhn, and
  Rustem]{wiesemann2013robust}
Wolfram Wiesemann, Daniel Kuhn, and Ber{\c{c}} Rustem.
\newblock Robust {M}arkov decision processes.
\newblock \emph{Mathematics of Operations Research}, 38\penalty0 (1):\penalty0
  153--183, 2013.

\bibitem[Lim and Autef(2019)]{lim2019kernel}
Shiau~Hong Lim and Arnaud Autef.
\newblock Kernel-based reinforcement learning in robust {M}arkov decision
  processes.
\newblock In \emph{Proc. International Conference on Machine Learning (ICML)},
  pages 3973--3981. PMLR, 2019.

\bibitem[Xu and Mannor(2010)]{xu2010distributionally}
Huan Xu and Shie Mannor.
\newblock Distributionally robust {M}arkov decision processes.
\newblock In \emph{Proc. Advances in Neural Information Processing Systems
  (NIPS)}, pages 2505--2513, 2010.

\bibitem[Yu and Xu(2015)]{yu2015distributionally}
Pengqian Yu and Huan Xu.
\newblock Distributionally robust counterpart in {M}arkov decision processes.
\newblock \emph{IEEE Transactions on Automatic Control}, 61\penalty0
  (9):\penalty0 2538--2543, 2015.

\bibitem[Lim et~al.(2013)Lim, Xu, and Mannor]{lim2013reinforcement}
Shiau~Hong Lim, Huan Xu, and Shie Mannor.
\newblock Reinforcement learning in robust {M}arkov decision processes.
\newblock In \emph{Proc. Advances in Neural Information Processing Systems
  (NIPS)}, pages 701--709, 2013.

\bibitem[Tamar et~al.(2014)Tamar, Mannor, and Xu]{tamar2014scaling}
Aviv Tamar, Shie Mannor, and Huan Xu.
\newblock Scaling up robust {MDP}s using function approximation.
\newblock In \emph{Proc. International Conference on Machine Learning (ICML)},
  pages 181--189. PMLR, 2014.

\bibitem[Roy et~al.(2017)Roy, Xu, and Pokutta]{roy2017reinforcement}
Aurko Roy, Huan Xu, and Sebastian Pokutta.
\newblock Reinforcement learning under model mismatch.
\newblock In \emph{Proc. Advances in Neural Information Processing Systems
  (NIPS)}, pages 3046--3055, 2017.

\bibitem[Zhou et~al.(2021)Zhou, Bai, Zhou, Qiu, Blanchet, and
  Glynn]{zhou2021finite}
Zhengqing Zhou, Qinxun Bai, Zhengyuan Zhou, Linhai Qiu, Jose Blanchet, and
  Peter Glynn.
\newblock Finite-sample regret bound for distributionally robust offline
  tabular reinforcement learning.
\newblock In \emph{Proc. International Conference on Artifical Intelligence and
  Statistics (AISTATS)}, pages 3331--3339. PMLR, 2021.

\bibitem[Yang et~al.(2021)Yang, Zhang, and Zhang]{yang2021towards}
Wenhao Yang, Liangyu Zhang, and Zhihua Zhang.
\newblock Towards theoretical understandings of robust {M}arkov decision
  processes: Sample complexity and asymptotics.
\newblock \emph{arXiv preprint arXiv:2105.03863}, 2021.

\bibitem[Panaganti and Kalathil(2021)]{panaganti2021sample}
Kishan Panaganti and Dileep Kalathil.
\newblock Sample complexity of robust reinforcement learning with a generative
  model.
\newblock \emph{arXiv preprint arXiv:2112.01506}, 2021.

\bibitem[Ho et~al.(2018)Ho, Petrik, and Wiesemann]{ho2018fast}
Chin~Pang Ho, Marek Petrik, and Wolfram Wiesemann.
\newblock Fast {B}ellman updates for robust {MDP}s.
\newblock In \emph{Proc. International Conference on Machine Learning (ICML)},
  pages 1979--1988. PMLR, 2018.

\bibitem[Ho et~al.(2021)Ho, Petrik, and Wiesemann]{ho2021partial}
Chin~Pang Ho, Marek Petrik, and Wolfram Wiesemann.
\newblock Partial policy iteration for l1-robust {M}arkov decision processes.
\newblock \emph{Journal of Machine Learning Research}, 22\penalty0
  (275):\penalty0 1--46, 2021.

\bibitem[Vinitsky et~al.(2020)Vinitsky, Du, Parvate, Jang, Abbeel, and
  Bayen]{vinitsky2020robust}
Eugene Vinitsky, Yuqing Du, Kanaad Parvate, Kathy Jang, Pieter Abbeel, and
  Alexandre Bayen.
\newblock Robust reinforcement learning using adversarial populations.
\newblock \emph{arXiv preprint arXiv:2008.01825}, 2020.

\bibitem[Pinto et~al.(2017)Pinto, Davidson, Sukthankar, and
  Gupta]{pinto2017robust}
Lerrel Pinto, James Davidson, Rahul Sukthankar, and Abhinav Gupta.
\newblock Robust adversarial reinforcement learning.
\newblock In \emph{Proc. International Conference on Machine Learning (ICML)},
  pages 2817--2826. PMLR, 2017.

\bibitem[Abdullah et~al.(2019)Abdullah, Ren, Ammar, Milenkovic, Luo, Zhang, and
  Wang]{abdullah2019wasserstein}
Mohammed~Amin Abdullah, Hang Ren, Haitham~Bou Ammar, Vladimir Milenkovic, Rui
  Luo, Mingtian Zhang, and Jun Wang.
\newblock Wasserstein robust reinforcement learning.
\newblock \emph{arXiv preprint arXiv:1907.13196}, 2019.

\bibitem[Hou et~al.(2020)Hou, Pang, Hong, Lan, Ma, and Yin]{hou2020robust}
Linfang Hou, Liang Pang, Xin Hong, Yanyan Lan, Zhiming Ma, and Dawei Yin.
\newblock Robust reinforcement learning with {W}asserstein constraint.
\newblock \emph{arXiv preprint arXiv:2006.00945}, 2020.

\bibitem[Rajeswaran et~al.(2017)Rajeswaran, Ghotra, Ravindran, and
  Levine]{rajeswaran2017epopt}
Aravind Rajeswaran, Sarvjeet Ghotra, Balaraman Ravindran, and Sergey Levine.
\newblock Epopt: Learning robust neural network policies using model ensembles.
\newblock In \emph{Proc. International Conference on Learning Representations
  (ICLR)}, 2017.

\bibitem[Huang et~al.(2017)Huang, Papernot, Goodfellow, Duan, and
  Abbeel]{huang2017adversarial}
Sandy Huang, Nicolas Papernot, Ian Goodfellow, Yan Duan, and Pieter Abbeel.
\newblock Adversarial attacks on neural network policies.
\newblock In \emph{Proc. International Conference on Learning Representations
  (ICLR)}, 2017.

\bibitem[Kos and Song(2017)]{kos2017delving}
Jernej Kos and Dawn Song.
\newblock Delving into adversarial attacks on deep policies.
\newblock In \emph{Proc. International Conference on Learning Representations
  (ICLR)}, 2017.

\bibitem[Lin et~al.(2017)Lin, Hong, Liao, Shih, Liu, and Sun]{lin2017tactics}
Yen-Chen Lin, Zhang-Wei Hong, Yuan-Hong Liao, Meng-Li Shih, Ming-Yu Liu, and
  Min Sun.
\newblock Tactics of adversarial attack on deep reinforcement learning agents.
\newblock In \emph{Proc. International Joint Conferences on Artificial
  Intelligence (IJCAI)}, pages 3756--3762, 2017.

\bibitem[Pattanaik et~al.(2018)Pattanaik, Tang, Liu, Bommannan, and
  Chowdhary]{pattanaik2018robust}
Anay Pattanaik, Zhenyi Tang, Shuijing Liu, Gautham Bommannan, and Girish
  Chowdhary.
\newblock Robust deep reinforcement learning with adversarial attacks.
\newblock In \emph{Proc. International Conference on Autonomous Agents and
  MultiAgent Systems}, pages 2040--2042, 2018.

\bibitem[Mandlekar et~al.(2017)Mandlekar, Zhu, Garg, Fei-Fei, and
  Savarese]{mandlekar2017adversarially}
Ajay Mandlekar, Yuke Zhu, Animesh Garg, Li~Fei-Fei, and Silvio Savarese.
\newblock Adversarially robust policy learning: Active construction of
  physically-plausible perturbations.
\newblock In \emph{2017 IEEE/RSJ International Conference on Intelligent Robots
  and Systems (IROS)}, pages 3932--3939. IEEE, 2017.

\bibitem[Ho and Ermon(2016)]{ho2016generative}
Jonathan Ho and Stefano Ermon.
\newblock Generative adversarial imitation learning.
\newblock \emph{Advances in neural information processing systems}, 29, 2016.

\bibitem[Fu et~al.(2017)Fu, Luo, and Levine]{fu2017learning}
Justin Fu, Katie Luo, and Sergey Levine.
\newblock Learning robust rewards with adversarial inverse reinforcement
  learning.
\newblock \emph{arXiv preprint arXiv:1710.11248}, 2017.

\bibitem[Torabi et~al.(2018)Torabi, Warnell, and Stone]{torabi2018generative}
Faraz Torabi, Garrett Warnell, and Peter Stone.
\newblock Generative adversarial imitation from observation.
\newblock \emph{arXiv preprint arXiv:1807.06158}, 2018.

\bibitem[Viano et~al.(2022)Viano, Huang, Kamalaruban, Innes, Ramamoorthy, and
  Weller]{viano2022robust}
Luca Viano, Yu-Ting Huang, Parameswaran Kamalaruban, Craig Innes, Subramanian
  Ramamoorthy, and Adrian Weller.
\newblock Robust learning from observation with model misspecification.
\newblock In \emph{Proceedings of the 21st International Conference on
  Autonomous Agents and Multiagent Systems}, pages 1337--1345, 2022.

\bibitem[Bertsekas(2014)]{bertsekas2014constrained}
Dimitri~P Bertsekas.
\newblock \emph{Constrained optimization and Lagrange multiplier methods}.
\newblock Academic press, 2014.

\bibitem[Puterman(2014)]{puterman2014Markov}
Martin~L Puterman.
\newblock \emph{{M}arkov decision processes: discrete stochastic dynamic
  programming}.
\newblock John Wiley \& Sons, 2014.

\bibitem[Agarwal et~al.(2021)Agarwal, Kakade, Lee, and
  Mahajan]{agarwal2021theory}
Alekh Agarwal, Sham~M Kakade, Jason~D Lee, and Gaurav Mahajan.
\newblock On the theory of policy gradient methods: Optimality, approximation,
  and distribution shift.
\newblock \emph{Journal of Machine Learning Research}, 22\penalty0
  (98):\penalty0 1--76, 2021.

\bibitem[Mei et~al.(2020)Mei, Xiao, Szepesvari, and Schuurmans]{mei2020global}
Jincheng Mei, Chenjun Xiao, Csaba Szepesvari, and Dale Schuurmans.
\newblock On the global convergence rates of softmax policy gradient methods.
\newblock In \emph{Proc. International Conference on Machine Learning (ICML)},
  pages 6820--6829. PMLR, 2020.

\bibitem[Li et~al.(2021{\natexlab{b}})Li, Wei, Chi, Gu, and
  Chen]{li2021softmax}
Gen Li, Yuting Wei, Yuejie Chi, Yuantao Gu, and Yuxin Chen.
\newblock Softmax policy gradient methods can take exponential time to
  converge.
\newblock \emph{arXiv preprint arXiv:2102.11270}, 2021{\natexlab{b}}.

\bibitem[Zhang et~al.(2021)Zhang, Tachet, and Laroche]{zhang2021global}
Shangtong Zhang, Remi Tachet, and Romain Laroche.
\newblock Global optimality and finite sample analysis of softmax off-policy
  actor critic under state distribution mismatch.
\newblock \emph{arXiv preprint arXiv:2111.02997}, 2021.

\bibitem[Wang and Zou(2020)]{wang2020finite}
Yue Wang and Shaofeng Zou.
\newblock Finite-sample analysis of {Greedy-GQ} with linear function
  approximation under {M}arkovian noise.
\newblock In \emph{Proc. International Conference on Uncertainty in Artificial
  Intelligence (UAI)}, pages 11--20. PMLR, 2020.

\bibitem[Du et~al.(2019)Du, Lee, Li, Wang, and Zhai]{du2019gradient}
Simon Du, Jason Lee, Haochuan Li, Liwei Wang, and Xiyu Zhai.
\newblock Gradient descent finds global minima of deep neural networks.
\newblock In \emph{Proc. International Conference on Machine Learning (ICML)},
  pages 1675--1685. PMLR, 2019.

\bibitem[Neyshabur(2017)]{neyshabur2017implicit}
Behnam Neyshabur.
\newblock Implicit regularization in deep learning.
\newblock \emph{arXiv preprint arXiv:1709.01953}, 2017.

\bibitem[Miyato et~al.(2018)Miyato, Kataoka, Koyama, and
  Yoshida]{miyato2018spectral}
Takeru Miyato, Toshiki Kataoka, Masanori Koyama, and Yuichi Yoshida.
\newblock Spectral normalization for generative adversarial networks.
\newblock In \emph{Proc. International Conference on Learning Representations
  (ICLR)}, 2018.

\bibitem[Lin et~al.(2020)Lin, Jin, and Jordan]{lin2020gradient}
Tianyi Lin, Chi Jin, and Michael Jordan.
\newblock On gradient descent ascent for nonconvex-concave minimax problems.
\newblock In \emph{Proc. International Conference on Machine Learning (ICML)},
  pages 6083--6093. PMLR, 2020.

\bibitem[Xu et~al.(2020)Xu, Zhang, Xu, and Lan]{xu2020unified}
Zi~Xu, Huiling Zhang, Yang Xu, and Guanghui Lan.
\newblock A unified single-loop alternating gradient projection algorithm for
  nonconvex-concave and convex-nonconcave minimax problems.
\newblock \emph{arXiv preprint arXiv:2006.02032}, 2020.

\bibitem[Beck(2017)]{beck2017first}
Amir Beck.
\newblock \emph{First-order methods in optimization}.
\newblock SIAM, 2017.

\bibitem[Clarke(1990)]{clarke1990optimization}
Frank~H Clarke.
\newblock \emph{Optimization and nonsmooth analysis}.
\newblock SIAM, 1990.

\bibitem[Russel et~al.(2021)Russel, Benosman, Van~Baar, and
  Corcodel]{russel2021lyapunov}
Reazul~Hasan Russel, Mouhacine Benosman, Jeroen Van~Baar, and Radu Corcodel.
\newblock Lyapunov robust constrained-{MDP}s: Soft-constrained robustly stable
  policy optimization under model uncertainty.
\newblock \emph{arXiv preprint arXiv:2108.02701}, 2021.

\bibitem[Archibald et~al.(1995)Archibald, McKinnon, and
  Thomas]{archibald1995generation}
TW~Archibald, KIM McKinnon, and LC~Thomas.
\newblock {On the generation of {M}arkov decision processes}.
\newblock \emph{Journal of the Operational Research Society}, 46\penalty0
  (3):\penalty0 354--361, 1995.

\bibitem[Brockman et~al.(2016)Brockman, Cheung, Pettersson, Schneider,
  Schulman, Tang, and Zaremba]{brockman2016openai}
Greg Brockman, Vicki Cheung, Ludwig Pettersson, Jonas Schneider, John Schulman,
  Jie Tang, and Wojciech Zaremba.
\newblock {OpenAI Gym}.
\newblock \emph{arXiv preprint arXiv:1606.01540}, 2016.

\bibitem[Ghadimi and Lan(2016)]{ghadimi2016accelerated}
Saeed Ghadimi and Guanghui Lan.
\newblock Accelerated gradient methods for nonconvex nonlinear and stochastic
  programming.
\newblock \emph{Mathematical Programming}, 156\penalty0 (1-2):\penalty0 59--99,
  2016.

\end{thebibliography}

\newpage
\appendix
\begin{large}
 \textbf{Appendix}
\end{large}


\section{Additional Experiments}
\textbf{$4\times 4$ Frozen Lake problem.}
The $4\times 4$ frozen lake is similar to the $8\times 8$ one but with a smaller map. Similarly, we randomly generate the utility signal for each state-action pair. 
The results are shown in Fig.\ref{Fig.f}.

\begin{figure}[ht]
\begin{center}
\subfigure[$V_c$ when $\delta=0.2$.]{
\begin{minipage}[t]{0.23\linewidth}
\centering
\label{Fig.f1c}
\includegraphics[width=1.47 in]{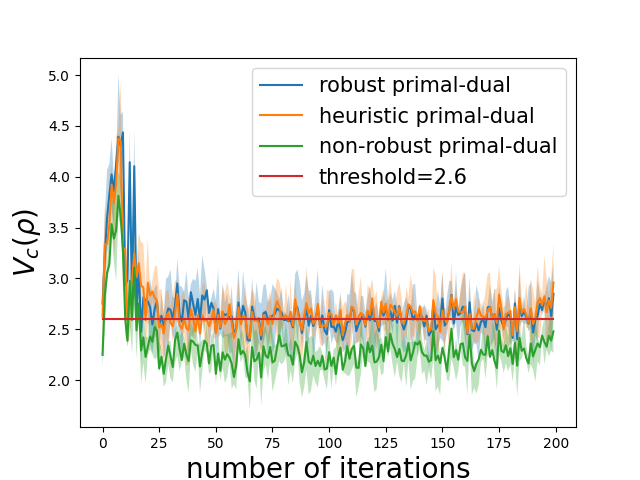}
\end{minipage}}
\subfigure[$V_r$ when $\delta=0.2$.]{
\begin{minipage}[t]{0.23\linewidth}
\centering
\label{Fig.f1r}
\includegraphics[width=1.47 in]{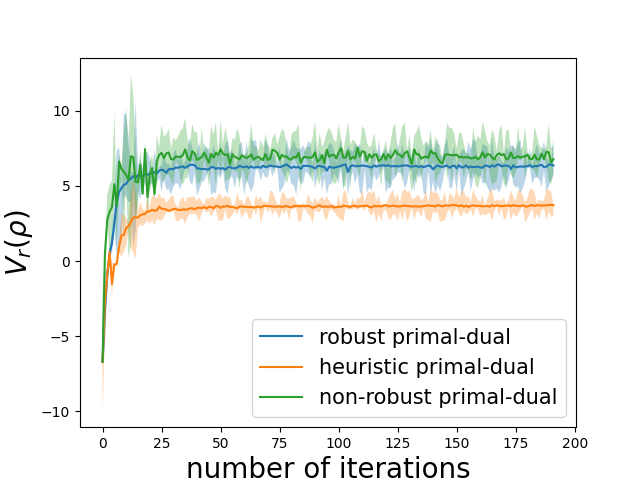}
\end{minipage}}
\subfigure[$V_c$ when $\delta=0.3$.]{
\begin{minipage}[t]{0.23\linewidth}
\centering
\label{Fig.f2c}
\includegraphics[width=1.47 in]{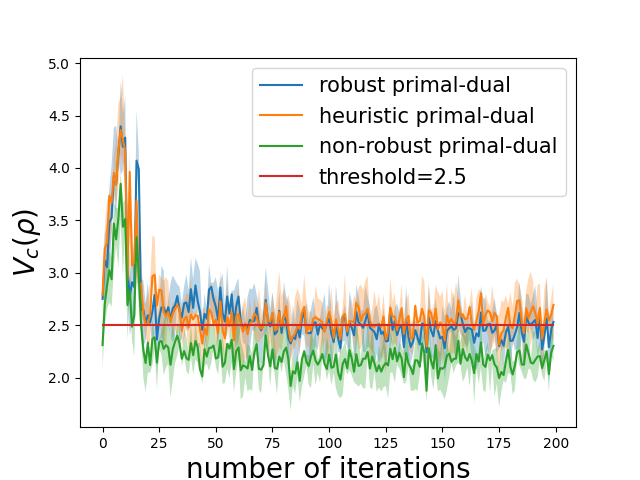}
\end{minipage}}
\subfigure[$V_r$ when $\delta=0.3$.]{
\begin{minipage}[t]{0.23\linewidth}
\centering
\label{Fig.f2r}
\includegraphics[width=1.47 in]{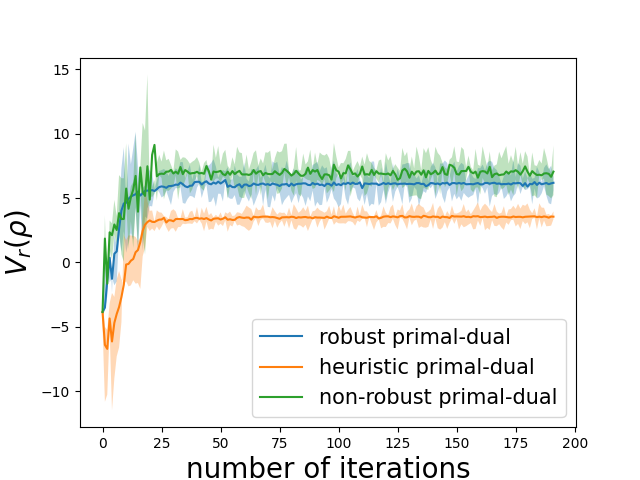}
\end{minipage}}
\caption{Comparison on $4\times 4$ Frozen-Lake Problem.}
\label{Fig.f}
\end{center}
\vskip -0.2in
\end{figure}

\textbf{$N$-Chain problem.}
We then compare three algorithms under the $N$-Chain problem environment. The $N$-chain problem involves a chain contains $N$ nodes. The agent can either move to its left or right node. When it goes to left, it receives a reward-utility signal $(1,0)$; When it goes right, it receives a reward-utility signal $(0,2)$, and if the agent arrives the $N$-th node, it receives a bonus reward of $40$. There is also a small probability that the agent slips to the different direction of its action. In this experiment, we set $N=40$.  The results are shown in Fig.\ref{Fig.nc}.

\begin{figure}[!h]
 \begin{center}
\subfigure[$V_c$ when $\delta=0.2$.]{
 \begin{minipage}[t]{0.23\linewidth}
 \centering
 \label{Fig.nc1c}
 \includegraphics[width=1.47 in]{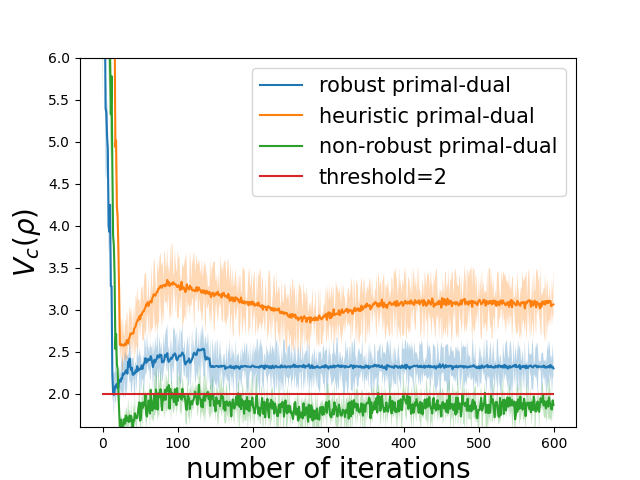}
 \end{minipage}}
 \subfigure[$V_r$ when $\delta=0.2$.]{
 \begin{minipage}[t]{0.23\linewidth}
 \centering
 \label{Fig.nc1r}
 \includegraphics[width=1.47 in]{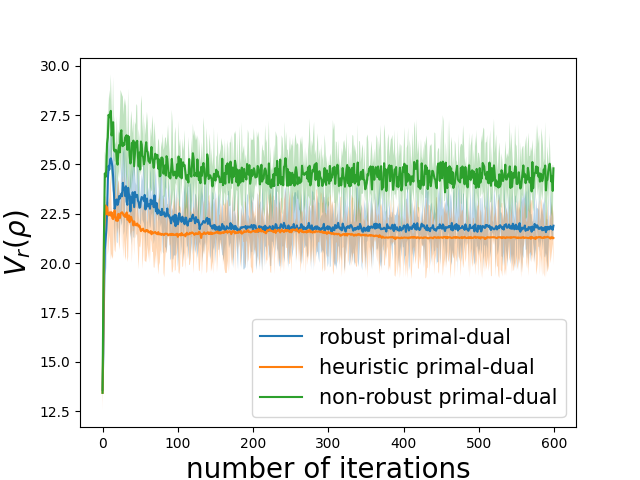}
 \end{minipage}}
 \subfigure[$V_c$ when $\delta=0.3$.]{
 \begin{minipage}[t]{0.23\linewidth}
 \centering
 \label{Fig.nc2c}
\includegraphics[width=1.47 in]{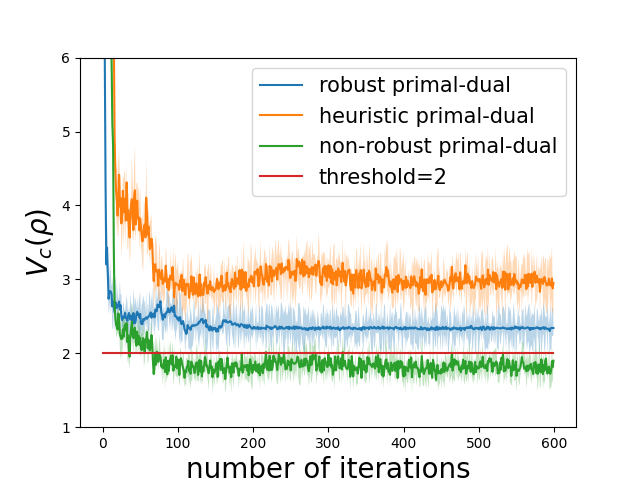}
 \end{minipage}}
\subfigure[$V_r$ when $\delta=0.3$.]{
 \begin{minipage}[t]{0.23\linewidth}
 \centering
 \label{Fig.nc2r}
\includegraphics[width=1.47 in]{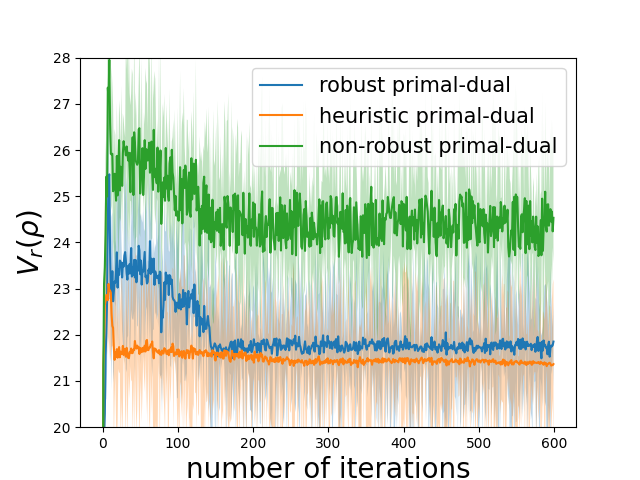}
\end{minipage}}
 \caption{Comparison on $N$-Chain Problem.}
\label{Fig.nc}
 \end{center}
 \vskip -0.2in
 \end{figure}

\section{Proof of Lemma \ref{lemma:fail duality}}\label{sec:counter}
Denote by $\mathsf P^\pi=\left\{(p^\pi)^a_s\in\Delta_{\mcs}: s\in\mcs,a\in\mca \right\}$ the worst-case transition kernel corresponding to the policy $\pi$. We consider the $\delta$-contamination uncertainty set defined in Section \ref{sec:ex}. We then show that under $\delta$-contamination model, the set of visitation distributions is non-convex. 
The robust visitation distribution set can be written as follows:
\begin{align}\label{eq:vis set1}
    \left\{d\in\Delta_{\mcs\times\mca}: \exists \pi\in\Pi, \text{ s.t. } \forall (s,a), \left\{
\begin{aligned}
 & d(s,a)=\pi(a|s)\sum_b d(s,b),&\\
 &\gamma \sum_{s',a'}({p^\pi})^{a'}_{s',s}d(s',a')+(1-\gamma)\rho(s)=\sum_a d(s,a).&
\end{aligned}
\right\}\right\}.
\end{align} 

Under the $\delta$-contamination model, ${\mathsf P}^\pi$ can be explicated as $(p^\pi)^{a}_{s,s'}=(1-\delta)p^a_{s,s'}+\delta \mathbbm{1}_{\left\{s'=\arg\min V^\pi\right\}}$. Hence the set in \eqref{eq:vis set1} can be rewritten as
\begin{align}
    \left\{d\in\Delta_{\mcs\times\mca}: \exists \pi, \text{ s.t. } \forall (s,a), \left\{
\begin{aligned}
 & d(s,a)=\pi(a|s)\left(\sum_b d(s,b)\right),&\\
 &\gamma(1-\delta)\sum_{s',a'}{p}^{a'}_{s',s}d(s',a')+\gamma \delta \mathbbm{1}_{\left\{s=\arg\min V^\pi\right\}}&\\
 &\quad\quad\quad\quad+(1-\gamma)\rho(s)=\sum_a d(s,a).&
\end{aligned}
\right\}\right\}.
\end{align}
Now consider any two pairs $(\pi_1,d_1), (\pi_2,d_2)$ of policy and their worst-case visitation distribution, to show that the set is convex, we need to find a pair $(\pi',d')$ such that  $\forall \lambda\in[0,1]$ and   $\forall s,a$,
\begin{align}
&\lambda d_1(s,a)+(1-\lambda) d_2(s,a)=d'(s,a),\label{eq:vis pi1}\\
&d'(s,a)=\pi'(a|s)\left(\sum_b d'(s,b)\right),\label{eq:vis pi}\\
&\sum_{a'} d'(s,a')=\gamma(1-\delta)\sum_{s',a'}{p}^{a'}_{s',s}d'(s',a')+\gamma \delta \mathbbm{1}_{\left\{s=\arg\min {V}^{\pi'}\right\}}+(1-\gamma)\rho(s) \label{eq:vis kernel}.
\end{align}
\eqref{eq:vis kernel} firstly implies that $\forall s$, 
\begin{align}\label{eq:222}
    \lambda \mathbbm{1}_{\left\{s=\arg\min {V}^{\pi_1}\right\}}+(1-\lambda) \mathbbm{1}_{\left\{s=\arg\min {V}^{\pi_2}\right\}}=\mathbbm{1}_{\left\{s=\arg\min {V}^{\pi'}\right\}},
\end{align}
where from \eqref{eq:vis pi1} and \eqref{eq:vis pi}, $\pi'$ should be \begin{align}\label{eq:convexpi}\pi'(a|s)=\frac{d'(s,a)}{\sum_b d'(s,b)}=\frac{\lambda d_1(s,a)+(1-\lambda) d_2(s,a)}{\sum_b(\lambda d_1(s,b)+(1-\lambda)d_2(s,b))}.\end{align}

We then construct the following counterexample, which shows that there exists a robust MDP, two policy-distribution pairs $(\pi_1,d_1), (\pi_2,d_2)$, and $\lambda\in(0,1)$, such that $\lambda \mathbbm{1}_{\left\{s=\arg\min {V}^{\pi_1}\right\}}+(1-\lambda) \mathbbm{1}_{\left\{s=\arg\min {V}^{\pi_2}\right\}}\neq \mathbbm{1}_{\left\{s=\arg\min {V}^{\pi'}\right\}}$, and therefore the set of robust visitation distribution is non-convex.  

Consider the following Robust MDP. It has three states $1,2,3$ and two actions $a,b$. When the agent is at state $1$, if it takes action $a$, the system will transit to state $2$ and receive reward $r=0$; if it takes action $b$, the system will   transit to state $3$ and receive reward $r=2$. When the agent is at state $2/3$, it can only take action $a/b$, the system can only   transits back to state $1$ and the agent will receive reward $r=1$. The initial distribution is $\mathbbm{1}_{s=1}$.

\begin{center}

\begin{tikzpicture}[
    node distance = 29mm and 17mm,
every edge/.style = {draw, -{Stealth[scale=1.2]}, bend left=15},
every edge quotes/.append style = {auto, inner sep=2pt, font=\footnotesize}
                        ]
\node (n1)  [state] {$1$};
\node (n3)  [state,below right=of n1]   {$3$};
\node (n2)  [state,above right=of n3]   {$2$};
\path   (n1)    edge ["action=$a \quad r=0$"] (n2)
                edge ["action=$b\quad r=2$"] (n3)
        (n2)    edge ["action=$a\quad  r=1$"] (n1)
        (n3)    edge ["action=$b \quad  r=1$"] (n1);
    \end{tikzpicture}
    
\end{center} 

Clearly all policy can be written as $\pi=(p,1-p)$, where $p$ is the probability of taking action $a$ at state $1$. We consider two policies, $\pi_1=(1,0)$ and $\pi_2=(0,1)$.

It can be verified that $\arg\min V^{\pi_1} =1$, and its robust visitation distribution, denoted by $d_1$, is 
\begin{align}
    &d_1(1,a)=\frac{1-\gamma}{1-\gamma^2},\\
    &d_1(1,b)=0,\\
    &d_1(2,a)=\frac{\gamma(1-\gamma)}{1-\gamma^2},\\
    &d_1(2,b)=0,\\
    &d_1(3,a)=0,\\
    &d_1(3,b)=0.
\end{align}

Similarly, $\arg\min V^{\pi_2} =2$, and and its robust visitation distribution, denoted by $d_2$, is \begin{align}
    &d_2(1,a)=0,\\
    &d_2(1,b)=\frac{1-\gamma}{1-\gamma^2},\\
    &d_2(2,a)=0,\\
    &d_2(2,b)=0,\\
    &d_2(3,a)=0,\\
    &d_2(3,b)=\frac{\gamma(1-\gamma)}{1-\gamma^2}.
\end{align}

Hence according to \eqref{eq:convexpi}, $\pi'$ should be as follows:
\begin{align}
    \pi'(a|1)=\lambda, \pi'(b|1)=1-\lambda, \pi'(a|2)=1, \pi'(b|3)=1.
\end{align}
We then show that there exists $\lambda\in[0,1]$, such that  $\lambda\mathbbm{1}_{\left\{s=1\right\}}+(1-\lambda)\mathbbm{1}_{\left\{s=2\right\}}\neq \mathbbm{1}_{\left\{\arg\min V^{\pi'}\right\}}$. 

Clearly \eqref{eq:222} holds only if $V^{\pi'}(1)=V^{\pi'}(2)=\min_s V^{\pi'}(s)$. However, according to the Bellman equations for $\pi'$, we have that 
\begin{align}
    V^{\pi'}(1)&=\lambda(\gamma (1-\delta) V^{\pi'}(2)+\gamma\delta \min V^{\pi'})+(1-\lambda)(2+\gamma (1-\delta) V^{\pi'}(3)+\gamma\delta \min V^{\pi'}),\\
    V^{\pi'}(2)&=1+\gamma(1-\delta) V^{\pi'}(1)+\gamma\delta \min V^{\pi'},\\
    V^{\pi'}(3)&=1+\gamma(1-\delta) V^{\pi'}(1)+\gamma\delta \min V^{\pi'}.
\end{align}

If we set $\lambda=\frac{1}{3}$, 
\begin{align}
    V^{\pi'}(1)&=\frac{4}{3}+\gamma\delta\min V^{\pi'}+\gamma(1-\delta)V^{\pi'}(2),\\
    V^{\pi'}(2)&=1+\gamma\delta\min V^{\pi'}+\gamma(1-\delta)V^{\pi'}(1).
\end{align}
Clearly, $V^{\pi'}(1)\neq V^{\pi'}(2)$, and hence $\lambda\mathbbm{1}_{\left\{\arg\min V^1\right\}}+(1-\lambda)\mathbbm{1}_{\left\{\arg\min V^2\right\}}\neq \mathbbm{1}_{\left\{\arg\min V^{\pi'}\right\}}$.

\section{Proof of Lemmas \ref{lemma:lam_bound} and \ref{lemma:222}}
\textbf{ Proof of Lemma \ref{lemma:lam_bound}}
\begin{proof} 
We first set $C=V^{\pi_{\theta^*}}_r(\rho)+\lambda^*(V_c^{\pi_{\theta^*}}(\rho)-b)$, clearly $\max_{\pi\in\Pi} V^\pi_r(\rho)+\lambda^*(V^\pi_c(\rho)-b)=C$, and hence 
\begin{align}
    C=\max_{\pi\in\Pi} V^\pi_r(\rho)+\lambda^*(V_c^\pi(\rho)-b)\geq V^{\pi^\zeta}_r(\rho)+\lambda^*(V_c^{\pi^\zeta}(\rho)-b)\geq V^{\pi^\zeta}_r(\rho)+\lambda^*\zeta.
\end{align}
Thus we have that 
\begin{align}
    \lambda^*\leq \frac{C-V^{\pi^\zeta}_r(\rho)}{\zeta}.
\end{align}
Note that \begin{align}
    C=\min_{\lambda\geq 0}\max_{\pi\in\Pi} V^\pi_r(\rho)+\lambda(V_c^\pi(\rho)-b)\overset{(a)}{\leq} \max_{\pi\in\Pi} V^\pi_r(\rho)\leq \frac{1}{1-\gamma},
\end{align}
where $(a)$ is because $\min_{\lambda\geq 0}\max_{\pi\in\Pi} V^\pi_r(\rho)+\lambda(V_c^\pi(\rho)-b)$ is less than the optimal value of inner problem when $\lambda=0$, i.e., $\max_{\pi\in\Pi} V^\pi_r(\rho)$, and $\frac{1}{1-\gamma}$ is the upper bound of robust value functions. Hence we have that
\begin{align}
    \lambda^*\leq \frac{1}{(1-\gamma)\zeta},
\end{align}
which completes the proof.
\end{proof}

\textbf{ Proof of Lemma \ref{lemma:222}}
\begin{proof}
Set $C=V^{_{\theta^*}}_{\sigma,r}(\rho)+\lambda^*(V_{\sigma,c}^{_{\theta^*}}(\rho)-b)$, then 
\begin{align}
    C=\max_{\pi\in\Pi} V^\pi_{\sigma,r}(\rho)+\lambda^*(V_{\sigma,c}^\pi(\rho)-b)\geq V^{\pi^{\zeta'}}_{\sigma,r}(\rho)+\lambda^*(V_{\sigma,c}^{\pi^{\zeta'}}(\rho)-b)\geq V^{\pi^{\zeta'}}_{\sigma,r}(\rho)+\lambda^*\zeta'.
\end{align}
Thus we have that 
\begin{align}
    C\geq V^{\pi^\zeta}_{\sigma,r}(\rho)+\lambda^*\zeta', 
\end{align}
hence
\begin{align}
    \lambda^*\leq \frac{C-V^{\pi^\zeta}_{\sigma,r}(\rho)}{\zeta'}.
\end{align}
Note that \begin{align}
    C=\min_{\lambda\geq 0}\max_{\pi\in\Pi} V^\pi_{\sigma,r}(\rho)+\lambda(V_{\sigma,c}^\pi(\rho)-b){\leq} \max_{\pi\in\Pi} V^\pi_{\sigma,r}(\rho)\leq C_\sigma,
\end{align}
 where $C_\sigma$ is the upper bound of smoothed robust value functions \cite{wang2022policy}:  $C_\sigma=\frac{1}{1-\gamma}(1+2\gamma R \frac{\log|\mcs|}{\sigma})$. Hence we have that
\begin{align}
    \lambda^*\leq \frac{C_\sigma}{\zeta'},
\end{align}
which completes the proof.
\end{proof}

\section{Proof of Lemma \ref{lemma2}}
\begin{proof}
For any $\lambda$, denote the optimal value of the inner problems $\max_{\pi\in\Pi_\Theta} V^\pi_{\sigma,r}(\rho)+\lambda(V_{\sigma,c}^\pi(\rho)-b)$ and $\max_{\pi\in\Pi_\Theta} V^\pi_{r}(\rho)+\lambda(V_c^\pi(\rho)-b)$ by $V^D(\lambda)$ and $V^D_\sigma(\lambda)$. It is then easy to verify that 
 \begin{align}
    | V^D(\lambda)-V^D_\sigma(\lambda) |\leq (1+\lambda)\epsilon\leq (1+\Lambda^*)\epsilon.
\end{align}
Denote the optimal solutions of $\min_{\lambda\in[0,\Lambda^*]} V^D(\lambda)$ and $\min_{\lambda\in[0,\Lambda^*]} V^D_\sigma(\lambda)$ by $\lambda^D$ and $\lambda^D_\sigma$. 
We thus conclude that $| V^D_\sigma(\lambda^D_\sigma)- V^D(\lambda^D)|\leq \left(1+\Lambda^*\right)\epsilon$, and this thus completes the proof.
\end{proof}

\section{Proof of Theorem \ref{thm:1}}
We  restate Theorem \ref{thm:1} with all the specific step sizes as follows. 

Set $b_t=\frac{19}{20\xi t^{0.25}}$, $\mu_t=\xi (C_\sigma^V)^2+\frac{16\tau(C_\sigma^V)^2 }{\xi (b_{t+1})^2}-2\nu,$
    $\beta_t=\frac{1}{\xi},
    \alpha_t=\nu+\mu_t$,  where $\xi>\frac{2\nu+(1+\Lambda^*)L_\sigma}{(C^V_\sigma)^2}$, $\nu$ is any positive number and $\tau$ is any number greater than $2$, then \begin{align}
    \min_{1\leq t\leq T}\|G_t\|^2\leq 2\epsilon,
\end{align}
when \begin{align}\label{eq:T}
    T=\max\left\{ \frac{7(\Lambda^*)^4}{\xi^4\epsilon^4},\left(2+\frac{9\xi(\tau-2)(C^V_\sigma)^2uK}{\epsilon^2}\right)^2\right\}=\mathcal{O}(\epsilon^{-4}).
\end{align}
The definitions of $u,K$ can be found in Section \ref{sec:constants}.

Theorem \ref{thm:1} can be proved similarly as  Theorem \ref{thm2}, and hence the proof is omitted here.

\section{Proof of Lemma \ref{lemma:lip}}
\begin{proof}
Recall that $V_\sigma^L(\theta,\lambda)= V^{\pit}_{\sigma,r}(\rho)+\lambda(V_{\sigma,c}^{\pit}(\rho)-b)$, hence we have that
\begin{align}
    \nabla_{\lambda} V_\sigma^L(\theta,\lambda)&=V_{\sigma,c}^{\pit}(\rho)-b,\\
    \nabla_{\theta} V_\sigma^L(\theta,\lambda)&=\nabla_{\theta}V^{\pit}_{\sigma,r}(\rho)+\lambda\nabla_\theta V_{\sigma,c}^{\pit}(\rho).
\end{align}

Note that in \cite{wang2022policy}, it has been shown that 
\begin{align}
    \|V^{\pone}_{\sigma,r}-V^{\ptwo}_{\sigma,r}\|&\leq C^V_\sigma \|\theta_1-\theta_2\|,\\
    \|\nabla_\theta V^{\pone}_{\sigma,r}-\nabla_\theta V^{\ptwo}_{\sigma,r}\|&\leq L_\sigma \|\theta_1-\theta_2\|,
\end{align}
where the definition of constants $C^V_\sigma$ and $L_\sigma$ can be found in Section \ref{sec:constants}.
Hence
\begin{align}
    \|\nabla_{\lambda} V_\sigma^L(\theta,\lambda)|_{\theta_1}-\nabla_{\lambda} V_\sigma^L(\theta,\lambda)|_{\theta_2} \|&=\|V_{\sigma,c}^{\pi_{\theta_1}}(\rho)-V_{\sigma,c}^{\pi_{\theta_2}}(\rho) \|\leq  C^V_\sigma \|\theta_1-\theta_2\|,\\
     \|\nabla_{\lambda} V_\sigma^L(\theta,\lambda)|_{\lambda_1}-\nabla_{\lambda} V_\sigma^L(\theta,\lambda)|_{\lambda_2} \|&=0.
\end{align}
Similarly, we have that
\begin{align}
    \|\nabla_{\theta} V_\sigma^L(\theta,\lambda)|_{\theta_1}-\nabla_{\theta} V_\sigma^L(\theta,\lambda)|_{\theta_2} \|&\leq  (1+\lambda)L_\sigma \|\theta_1-\theta_2\|\leq (1+\Lambda^*)L_\sigma\|\theta_1-\theta_2\|,\\
     \|\nabla_{\theta} V_\sigma^L(\theta,\lambda)|_{\lambda_1}-\nabla_{\theta} V_\sigma^L(\theta,\lambda)|_{\lambda_2} \|&\leq  |(\lambda_1-\lambda_2)| \max_{\theta\in\Theta}\|\nabla_\theta V_{\sigma,c}^{\pit}(\rho) \| \leq C^V_\sigma |\lambda_1-\lambda_2|.
\end{align}
This completes the proof.
\end{proof}

\section{Proof of Proposition \ref{prop:feasible}}
\begin{proof}
The $\lambda$-entry of $G_W$ is smaller than $2\epsilon$, i.e.,
\begin{align}
    |(G_W)_\lambda|=\bigg|\beta_W \bigg(\lambda_W-\mathbf{\prod}_{[0,\Lambda^*]} \big(\lambda_W -\frac{1}{\beta_W} \big(\nabla_\lambda V^L_\sigma (\theta_W,\lambda_W) \big)\big)\bigg)\bigg|<2\epsilon.\end{align}
Denote $\lambda^+\triangleq \mathbf{\prod}_{[0,\Lambda^*]} \left(\lambda_W -\frac{1}{\beta_W} \left(\nabla_\lambda V^L_\sigma (\theta_W,\lambda_W) \right)\right)$. From Lemma 3 in \cite{ghadimi2016accelerated}, $-\nabla_\lambda V^L_\sigma (\theta_W,\lambda^+)$ can be rewritten as the sum of two parts: $-\nabla_\lambda V^L_\sigma (\theta_W,\lambda^+)\in N_{[0,\Lambda^*]}(\lambda^+)+4\epsilon B$, where $N_K(x)\triangleq\left\{g\in \mathbb{R}^d: \langle g, y-x\rangle\leq 0: \forall y\in K  \right\}$ is the normal cone, and $B$ is the unit ball. 

This hence implies that for any $\lambda\in[0,\Lambda^*]$, $(\lambda-\lambda^+)(V_c^W-b)\geq -4(\lambda-\lambda^+)\epsilon$. By setting $\lambda=\Lambda^*$, we have $V_c^W+4\epsilon\geq b$, which means $\pi_W$ is feasible with a $4\epsilon$-violation.
\end{proof}

\section{Proof of Theorem \ref{thm2}}
We then prove Theorem \ref{thm2}. Our proof extends the one in \cite{xu2020unified} to the biased setting.

To simplify notations, we denote the updates in Algorithm \ref{alg:srpd} by $ \hat{f}(\theta_t)\triangleq \hat{V}^{\pi_{\theta_t}}_{\sigma,c}(\rho)-b $, and $\hat{g}(\theta_t,\lambda_{t+1})\triangleq\nabla_{\theta}\hat{V}^{\pi_{\theta_t}}_{\sigma,r}(\rho)+\lambda_{t+1}\nabla_\theta \hat{V}^{\pi_{\theta_t}}_{\sigma,c}(\rho)  $, and denote the update functions in Algorithm \ref{alg:rpd} by $ {f}(\theta_t)\triangleq {V}^{\pi_{\theta_t}}_{\sigma,c}(\rho)-b $, and ${g}(\theta_t,\lambda_{t+1})\triangleq\nabla_{\theta}{V}^{\pi_{\theta_t}}_{\sigma,r}(\rho)+\lambda_{t+1}\nabla_\theta {V}^{\pi_{\theta_t}}_{\sigma,c}(\rho)$. Here $\hat{f}$ and $\hat{g}$ can be viewed as biased estimations of $f$ and $g$. 

In the following, we will first show several technical lemmas that will be useful in the proof of Theorem \ref{thm2}.

\begin{lemma}\label{lemma4}
Recall that the step size $\alpha_t=\nu+\mu_t$. If $\mu_t>(1+\Lambda^*)L_\sigma$, $\forall t\geq 0$, then 
\begin{align}
     V^L_{\sigma}(\theta_{t+1},\lambda_{t+1})-V^L_{\sigma}(\theta_{t},\lambda_{t+1}) &\geq \langle\theta_{t+1}-\theta_t, -\hat{g}(\theta_{t},\lambda_{t+1})+ g(\theta_{t},\lambda_{t+1}) \rangle\nn\\
     &\quad+\left(  \frac{\mu_t}{2}+\nu  \right)\|\theta_{t+1}-\theta_t\|^2.
\end{align}
\end{lemma}
\begin{proof}
Note that from the update of $\theta_t$ and proposition of projection, it implies that
\begin{align}
    \left\langle\theta_t+\frac{1}{\alpha_t} \hat{g}(\theta_t,\lambda_{t+1})-\theta_{t+1}, \theta_t-\theta_{t+1}  \right\rangle\leq 0. 
\end{align}
Hence 
\begin{align}
    \left\langle  \hat{g}(\theta_t,\lambda_{t+1})-\alpha_t(\theta_{t+1}-\theta_t), \theta_t-\theta_{t+1}  \right\rangle\leq 0.
\end{align}
From Lemma \ref{lemma:lip}, we have that
\begin{align}
    V^L_{\sigma}(\theta_{t+1},\lambda_{t+1})-V^L_{\sigma}(\theta_{t},\lambda_{t+1})\geq \langle\theta_{t+1}-\theta_t, g(\theta_{t},\lambda_{t+1}) \rangle-\frac{(1+\Lambda^*)L_\sigma}{2}\|\theta_{t+1}-\theta_t\|^2.
\end{align}
Summing up the two inequalities implies 
\begin{align}
    &V^L_{\sigma}(\theta_{t+1},\lambda_{t+1})-V^L_{\sigma}(\theta_{t},\lambda_{t+1})\nn\\
    &\geq \langle\theta_{t+1}-\theta_t, -\hat{g}(\theta_{t},\lambda_{t+1})+ g(\theta_{t},\lambda_{t+1})+\alpha_t(\theta_{t+1}-\theta_t) \rangle-\frac{(1+\Lambda^*)L_\sigma}{2}\|\theta_{t+1}-\theta_t\|^2\nn\\
    &\geq \langle\theta_{t+1}-\theta_t, -\hat{g}(\theta_{t},\lambda_{t+1})+ g(\theta_{t},\lambda_{t+1}) \rangle + \left( {\alpha_t} -\frac{L_\sigma(1+\Lambda^*)}{2} \right)\|\theta_{t+1}-\theta_t\|^2\nn\\
    &\geq \langle\theta_{t+1}-\theta_t, -\hat{g}(\theta_{t},\lambda_{t+1})+ g(\theta_{t},\lambda_{t+1}) \rangle +\left(  \frac{\mu_t}{2}+\nu  \right)\|\theta_{t+1}-\theta_t\|^2,
\end{align}
and hence completes the proof.
\end{proof}

\begin{lemma}\label{lemma5}
Recall that the step size $\beta_t=\frac{1}{\xi}$, and set $\xi\leq \frac{1}{b_0}$ , then
\begin{align}
     &V^L_{\sigma}(\theta_{t+1},\lambda_{t+1})-V^L_{\sigma}(\theta_{t},\lambda_{t}) \nn\\
     &\geq ( f(\theta_{t-1})-\hat{f}(\theta_{t-1}))( \lambda_{t+1}-\lambda_t)+\langle\theta_{t+1}-\theta_t, -\hat{g}(\theta_{t},\lambda_{t+1})+ g(\theta_{t},\lambda_{t+1}) \rangle  - \frac{\xi (C_\sigma^V)^2}{2} \|\theta_t-\theta_{t-1}\|^2\nn\\
     &\quad+\left(  \frac{\mu_t}{2}+\nu  \right)\|\theta_{t+1}-\theta_t\|^2+\frac{b_{t-1}}{2}(\lambda_t^2-\lambda_{t+1}^2)-\frac{1}{\xi}(\lambda_{t+1}-\lambda_t)^2-\frac{1}{2\xi}(\lambda_t-\lambda_{t-1})^2.
\end{align}
\end{lemma}

\begin{proof}
For any $t>1$, define $\tilde{V}_t(\theta,\lambda)\triangleq  V^L_{\sigma}(\theta,\lambda)+\frac{b_{t-1}}{2} \lambda^2$. Thus we have 
\begin{align}\label{eq:72}
    |\nabla_\lambda \tilde{V}_t(\theta_{t},\lambda_{t+1})-\nabla_\lambda\tilde{V}_t(\theta_t,\lambda_t)|= b_{t-1}|\lambda_{t+1}-\lambda_{t}|\leq b_0|\lambda_{t+1}-\lambda_{t}|,
\end{align}
where that last inequality is due to $b_{t-1}\leq b_0$. Note that $\tilde{V}_t(\theta,\lambda)$ is $b_{t-1}$-strongly convex in $\lambda$, hence we have
\begin{align}\label{eq:44}
    &(\nabla_\lambda \tilde{V}_t(\theta,\lambda_{t+1})-\nabla_\lambda \tilde{V}_t(\theta,\lambda_{t}))( \lambda_{t+1}-\lambda_t )\nn\\
    &\geq b_{t-1} (\lambda_{t+1}-\lambda_t)^2\nn\\
    &\geq b_{t-1}\left( \frac{b_{t-1}+b_0}{b_{t-1}+b_0} \right)(\lambda_{t+1}-\lambda_t)^2\nn\\
    &=\frac{b_{t-1}b_0}{b_{t-1}+b_0}(\lambda_{t+1}-\lambda_t)^2+ \frac{b_{t-1}^2}{b_{t-1}+b_0}(\lambda_{t+1}-\lambda_t)^2\nn\\
    &\geq \frac{b_{t-1}b_0}{b_{t-1}+b_0}(\lambda_{t+1}-\lambda_t)^2+ \frac{1}{b_{t-1}+b_0}(\nabla_\lambda \tilde{V}_t(\theta_{t},\lambda_{t+1})-\nabla_\lambda\tilde{V}_t(\theta_t,\lambda_t))^2,
\end{align}
where the last inequality is from \eqref{eq:72}.

Recall the update of $\lambda_t$ in Algorithm \ref{alg:srpd} which can be rewritten as
\begin{align}
\lambda_{t+1}=\mathbf{\prod}_{[0,\Lambda^*]} \left(\lambda_t -\frac{1}{\beta_t} \nabla_\lambda \tilde{V}_{t+1}(\theta_t, \lambda_t)+\frac{1}{\beta_t} (f(\theta_t)-\hat{f}(\theta_t)) \right),
\end{align}
This further implies that   $\forall \lambda\in[0,\Lambda^*]$:
\begin{align}
    (\beta_t(\lambda_{t+1}-\lambda_t)+ \nabla_\lambda \tilde{V}_{t+1}(\theta_t, \lambda_t)-f(\theta_t)+\hat{f}(\theta_t))( \lambda-\lambda_{t+1})\geq 0.
\end{align}
Hence setting $\lambda=\lambda_k$ implies that
\begin{align}\label{eq:47}
    (\beta_t(\lambda_{t+1}-\lambda_t)+ \nabla_\lambda \tilde{V}_{t+1}(\theta_t, \lambda_t)-f(\theta_t)+\hat{f}(\theta_t))( \lambda_t-\lambda_{t+1})\geq 0.
\end{align}
Similarly, we have that
\begin{align}\label{eq:48}
    ( \beta_t(\lambda_{t}-\lambda_{t-1})+ \nabla_\lambda \tilde{V}_{t}(\theta_{t-1}, \lambda_{t-1})-f(\theta_{t-1})+\hat{f}(\theta_{t-1}))( \lambda_{t+1}-\lambda_{t})\geq 0.
\end{align}
Note that $\tilde{V}_{t}$ is convex, we hence have that
\begin{align}\label{eq:49}
    &\tilde{V}_{t}(\theta_t, \lambda_{t+1}) -\tilde{V}_{t}(\theta_t, \lambda_{t})\nn\\
    &\geq (\nabla_\lambda \tilde{V}_{t}(\theta_t, \lambda_{t}))( \lambda_{t+1}-\lambda_t)\nn\\
    &=(\nabla_\lambda \tilde{V}_{t}(\theta_t, \lambda_{t})-\nabla_\lambda \tilde{V}_{t}(\theta_{t-1}, \lambda_{t-1}))( \lambda_{t+1}-\lambda_t) + (\nabla_\lambda \tilde{V}_{t}(\theta_{t-1}, \lambda_{t-1}))( \lambda_{t+1}-\lambda_t) \nn\\
    &\overset{(a)}{\geq} (\nabla_\lambda \tilde{V}_{t}(\theta_t, \lambda_{t})-\nabla_\lambda \tilde{V}_{t}(\theta_{t-1})( \lambda_{t-1}), \lambda_{t+1}-\lambda_t)\nn\\
    &\quad+ ( f(\theta_{t-1})-\hat{f}(\theta_{t-1})-\beta_t (\lambda_t-\lambda_{t-1}))( \lambda_{t+1}-\lambda_t),
\end{align}
where $(a)$ is from \eqref{eq:48}. The first term in the RHS of \eqref{eq:49} can be further bounded as follows.
\begin{align}
    &( \nabla_\lambda \tilde{V}_{t}(\theta_t, \lambda_{t})-\nabla_\lambda \tilde{V}_{t}(\theta_{t-1}, \lambda_{t-1}))( \lambda_{t+1}-\lambda_t)\nn\\
    &=( \nabla_\lambda \tilde{V}_{t}(\theta_t, \lambda_{t})-\nabla_\lambda \tilde{V}_{t}(\theta_{t-1}, \lambda_{t}))( \lambda_{t+1}-\lambda_t) \nn\\
    &\quad +( \nabla_\lambda \tilde{V}_{t}(\theta_{t-1}, \lambda_{t})-\nabla_\lambda \tilde{V}_{t}(\theta_{t-1}, \lambda_{t-1}))( \lambda_{t+1}-\lambda_t)\nn\\
    &=( \nabla_\lambda \tilde{V}_{t}(\theta_t, \lambda_{t})-\nabla_\lambda \tilde{V}_{t}(\theta_{t-1}, \lambda_{t}))( \lambda_{t+1}-\lambda_t) \nn\\
    &\quad +( \nabla_\lambda \tilde{V}_{t}(\theta_{t-1}, \lambda_{t})-\nabla_\lambda \tilde{V}_{t}(\theta_{t-1}, \lambda_{t-1}))( \lambda_{t}-\lambda_{t-1})\nn\\
    &\quad +m_{t+1} ( \nabla_\lambda \tilde{V}_{t}(\theta_{t-1}, \lambda_{t})-\nabla_\lambda \tilde{V}_{t}(\theta_{t-1}, \lambda_{t-1})),
\end{align}
where $m_{t+1}\triangleq (\lambda_{t+1}-\lambda_t)-(\lambda_t-\lambda_{t-1})$. Plug it in \eqref{eq:49} and we have that
\begin{align}\label{eq:51}
    &\tilde{V}_{t}(\theta_t, \lambda_{t+1}) -\tilde{V}_{t}(\theta_t, \lambda_{t})\nn\\
    &\geq ( f(\theta_{t-1})-\hat{f}(\theta_{t-1})-\beta_t (\lambda_t-\lambda_{t-1}))( \lambda_{t+1}-\lambda_t )\nn\\
    &\quad + \underbrace{(\nabla_\lambda \tilde{V}_{t}(\theta_t, \lambda_{t})-\nabla_\lambda \tilde{V}_{t}(\theta_{t-1}, \lambda_{t}))( \lambda_{t+1}-\lambda_t)}_{(a)} \nn\\
    &\quad +\underbrace{(\nabla_\lambda \tilde{V}_{t}(\theta_{t-1}, \lambda_{t})-\nabla_\lambda \tilde{V}_{t}(\theta_{t-1}, \lambda_{t-1}))( \lambda_{t}-\lambda_{t-1})}_{(b)}\nn\\
    &\quad +\underbrace{(\nabla_\lambda \tilde{V}_{t}(\theta_{t-1}, \lambda_{t})-\nabla_\lambda \tilde{V}_{t}(\theta_{t-1}, \lambda_{t-1})) m_{t+1}}_{(c)}.
\end{align}
We then provide bounds for each term in \eqref{eq:51} as follows.

Term $(a)$ can be bounded as follows: 
\begin{align}\label{eq:52}
    &(\nabla_\lambda \tilde{V}_{t}(\theta_t, \lambda_{t})-\nabla_\lambda \tilde{V}_{t}(\theta_{t-1}, \lambda_{t}))( \lambda_{t+1}-\lambda_t)\nn\\
    &=(\nabla_\lambda {V}_{\sigma}^L(\theta_t, \lambda_{t})-\nabla_\lambda {V}_{\sigma}^L(\theta_{t-1}, \lambda_{t}))( \lambda_{t+1}-\lambda_t)\nn\\
    &\geq \frac{-(\lambda_{t+1}-\lambda_t)^2}{2\xi}- \frac{\xi}{2} (\nabla_\lambda {V}_{\sigma}^L(\theta_t, \lambda_{t})-\nabla_\lambda {V}_{\sigma}^L(\theta_{t-1}, \lambda_{t}) )^2\nn\\
    &\geq \frac{-(\lambda_{t+1}-\lambda_t)^2}{2\xi}- \frac{\xi (C_\sigma^V)^2}{2} \|\theta_t-\theta_{t-1}\|^2,
\end{align}
which is from Cauchy–Schwarz inequality and $C_\sigma^V$-smoothness of ${V}_{\sigma}^L(\theta, \lambda)$.
 
 Term $(b)$ can be bounded as follows:
\begin{align}\label{eq:54}
    &(\nabla_\lambda \tilde{V}_{t}(\theta_{t-1}, \lambda_{t})-\nabla_\lambda \tilde{V}_{t}(\theta_{t-1}, \lambda_{t-1}))( \lambda_{t}-\lambda_{t-1})\nn\\
    &\geq \frac{1}{b_{t-1}+b_0}(\nabla_\lambda \tilde{V}_t(\theta_{t-1},\lambda_{t})-\nabla_\lambda\tilde{V}_t(\theta_{t-1},\lambda_{t-1}) )^2,
\end{align}
which is from \eqref{eq:44}.

Term $(c)$ can be bounded as follows by Cauchy–Schwarz inequality:
\begin{align}\label{eq:53}
 &m_{t+1}( \nabla_\lambda \tilde{V}_{t}(\theta_{t-1}, \lambda_{t})-\nabla_\lambda \tilde{V}_{t}(\theta_{t-1}, \lambda_{t-1}))\nn\\
 &\geq -\frac{\xi}{2} (\nabla_\lambda \tilde{V}_{t}(\theta_{t-1}, \lambda_{t})-\nabla_\lambda \tilde{V}_{t}(\theta_{t-1},\lambda_{t-1}))^2-\frac{1}{2\xi} m_{t+1}^2
\end{align}

Moreover, it can be shown that 
\begin{align}\label{eq:55}
    &\frac{1}{\xi}(\lambda_{t+1}-\lambda_t)( \lambda_t-\lambda_{t-1}) =\frac{1}{2\xi}(\lambda_{t+1}-\lambda_t)^2+\frac{1}{2\xi}(\lambda_t-\lambda_{t-1})^2-\frac{1}{2\xi}m_{t+1}^2.
\end{align}
Plug \eqref{eq:52} to \eqref{eq:55} in \ref{eq:51}, and we have that
\begin{align}\label{eq:57}
    &\tilde{V}_{t}(\theta_t, \lambda_{t+1}) -\tilde{V}_{t}(\theta_t, \lambda_{t})\nn\\
    &\geq ( f(\theta_{t-1})-\hat{f}(\theta_{t-1}))( \lambda_{t+1}-\lambda_t ) -\beta_t (\lambda_t-\lambda_{t-1})( \lambda_{t+1}-\lambda_t )\nn\\
    &\quad + (\nabla_\lambda \tilde{V}_{t}(\theta_t, \lambda_{t})-\nabla_\lambda \tilde{V}_{t}(\theta_{t-1}, \lambda_{t}))( \lambda_{t+1}-\lambda_t)  +(\nabla_\lambda \tilde{V}_{t}(\theta_{t-1}, \lambda_{t})-\nabla_\lambda \tilde{V}_{t}(\theta_{t-1}, \lambda_{t-1}))( \lambda_{t}-\lambda_{t-1})\nn\\
    &\quad + m_{t+1}( \nabla_\lambda \tilde{V}_{t}(\theta_{t-1}, \lambda_{t})-\nabla_\lambda \tilde{V}_{t}(\theta_{t-1}, \lambda_{t-1}))\nn\\
    &\geq ( f(\theta_{t-1})-\hat{f}(\theta_{t-1}))( \lambda_{t+1}-\lambda_t ) -\frac{1}{2\xi}(\lambda_{t+1}-\lambda_t)^2-\frac{1}{2\xi}(\lambda_t-\lambda_{t-1})^2+\frac{1}{2\xi}m_{t+1}^2\nn\\
    &\quad-\frac{(\lambda_{t+1}-\lambda_t)^2}{2\xi}- \frac{\xi (C_\sigma^V)^2}{2} \|\theta_t-\theta_{t-1}\|^2 +\frac{1}{b_{t-1}+b_0}(\nabla_\lambda \tilde{V}_t(\theta_{t-1},\lambda_{t})-\nabla_\lambda\tilde{V}_t(\theta_{t-1},\lambda_{t-1}) )^2\nn\\
    &\quad -\frac{\xi}{2} (\nabla_\lambda \tilde{V}_{t}(\theta_{t-1}, \lambda_{t})-\nabla_\lambda \tilde{V}_{t}(\theta_{t-1},\lambda_{t-1}))^2-\frac{1}{2\xi} m_{t+1}^2\nn\\
    &\geq ( f(\theta_{t-1})-\hat{f}(\theta_{t-1}))( \lambda_{t+1}-\lambda_t ) -\frac{1}{\xi}(\lambda_{t+1}-\lambda_t)^2-\frac{1}{2\xi}(\lambda_t-\lambda_{t-1})^2 - \frac{\xi (C_\sigma^V)^2}{2} \|\theta_t-\theta_{t-1}\|^2.
\end{align}
From the definition of $\tilde{V}_t$, we have that 
\begin{align}
    &\tilde{V}_{t}(\theta_t, \lambda_{t+1}) -\tilde{V}_{t}(\theta_t, \lambda_{t})\nn\\
    &=V^L_\sigma(\theta_t, \lambda_{t+1})+\frac{b_{t-1}}{2} \lambda_{t+1}^2- V^L_\sigma(\theta_t, \lambda_{t})-\frac{b_{t-1}}{2} \lambda_{t}^2.
\end{align}
Then we have that
\begin{align}
    &V^L_\sigma(\theta_t, \lambda_{t+1}) - V^L_\sigma(\theta_t, \lambda_{t}) \nn\\
    &\geq \frac{b_{t-1}}{2}(\lambda_t^2-\lambda_{t+1}^2)+( f(\theta_{t-1})-\hat{f}(\theta_{t-1}))( \lambda_{t+1}-\lambda_t )\nn\\
    &\quad -\frac{1}{\xi}(\lambda_{t+1}-\lambda_t)^2-\frac{1}{2\xi}(\lambda_t-\lambda_{t-1})^2 - \frac{\xi (C_\sigma^V)^2}{2} \|\theta_t-\theta_{t-1}\|^2.
\end{align}
Combining with Lemma \ref{lemma4}, if $\forall t$, $\mu_t>(1+\Lambda^*)L_\sigma$, we then have that  
\begin{align}
     &V^L_{\sigma}(\theta_{t+1},\lambda_{t+1})-V^L_{\sigma}(\theta_{t},\lambda_{t}) \nn\\
     &\geq ( f(\theta_{t-1})-\hat{f}(\theta_{t-1}))( \lambda_{t+1}-\lambda_t )+\langle\theta_{t+1}-\theta_t, -\hat{g}(\theta_{t},\lambda_{t+1})+ g(\theta_{t},\lambda_{t+1}) \rangle  - \frac{\xi (C_\sigma^V)^2}{2} \|\theta_t-\theta_{t-1}\|^2\nn\\
     &\quad+\left(  \frac{\mu_t}{2}+\nu  \right)\|\theta_{t+1}-\theta_t\|^2+\frac{b_{t-1}}{2}(\lambda_t^2-\lambda_{t+1}^2)-\frac{1}{\xi}(\lambda_{t+1}-\lambda_t)^2-\frac{1}{2\xi}(\lambda_t-\lambda_{t-1})^2.
\end{align}
\end{proof}

\begin{lemma}
Define
\begin{align}
    F_{t+1}&\triangleq-\frac{8}{\xi^2b_{t+1}}(\lambda_t-\lambda_{t+1} )^2-\frac{8}{\xi}\left(1-\frac{b_{t}}{b_{t+1}}\right)\lambda_{t+1}^2 + V^L_{\sigma}(\theta_{t+1},\lambda_{t+1})+ \frac{b_t}{2}\lambda_{t+1}^2 \nn\\
    &\quad +\left(-\frac{ 16(C^V_\sigma)^2}{ \xi b^2_{t+1}}- \frac{\xi (C_\sigma^V)^2}{2} \right)\|\theta_{t+1}-\theta_{t}\|^2   +\left(\frac{8}{\xi}-\frac{1}{2\xi}\right)(\lambda_{t+1}-\lambda_{t})^2,
\end{align}
and if $\frac{1}{b_{t+1}}-\frac{1}{b_{t}}\leq \frac{\xi}{5}$, then 
\begin{align}
    &F_{t+1}-F_t\nn\\
    &\geq S_t+\left(\frac{\mu_t}{2}+\nu -\frac{ 16(C^V_\sigma)^2}{ \xi b^2_{t+1}}- \frac{\xi (C_\sigma^V)^2}{2} \right)\|\theta_{t+1}-\theta_{t}\|^2 + \frac{b_t-b_{t-1}}{2}\lambda_{t+1}^2\nn\\
    &\quad+\frac{9}{10\xi}(\lambda_{t+1}-\lambda_{t})^2+\frac{8}{\xi}\left(\frac{b_{t}}{b_{t+1}}-\frac{b_{t-1}}{b_t}\right)\lambda_{t+1}^2,
\end{align}
where $S_t\triangleq \frac{16}{b_t\xi}( f(\theta_{t-1})-\hat{f}(\theta_{t-1})-f(\theta_t)+\hat{f}(\theta_t))( -\lambda_t+\lambda_{t+1})+ ( f(\theta_{t-1})-\hat{f}(\theta_{t-1}))( \lambda_{t+1}-\lambda_t )+\langle\theta_{t+1}-\theta_t, -\hat{g}(\theta_{t},\lambda_{t+1})+ g(\theta_{t},\lambda_{t+1}) \rangle$.
\end{lemma}
\begin{proof}
From \eqref{eq:47} and \eqref{eq:48}, we have that
\begin{align}\label{eq:61}
    &\beta_t m_{t+1}( \lambda_t-\lambda_{t+1})\geq (\nabla_\lambda \tilde{V}_{t+1}(\theta_t, \lambda_t)-\nabla_\lambda \tilde{V}_{t}(\theta_{t-1}, \lambda_{t-1}))( -\lambda_t+\lambda_{t+1})\nn\\
    &\quad+( f(\theta_{t-1})-\hat{f}(\theta_{t-1})-f(\theta_t)+\hat{f}(\theta_t))( -\lambda_t+\lambda_{t+1}).
\end{align}
The first term can be rewritten as
\begin{align}\label{eq:rewritten}
    &(\nabla_\lambda \tilde{V}_{t+1}(\theta_t, \lambda_t)-\nabla_\lambda \tilde{V}_{t}(\theta_{t-1}, \lambda_{t-1}))(\lambda_{t+1}-\lambda_t)\nn\\
    &=(\nabla_\lambda \tilde{V}_{t+1}(\theta_t, \lambda_t)-\nabla_\lambda\tilde{V}_{t}(\theta_{t-1}, \lambda_{t}))( \lambda_{t+1}-\lambda_t)+(\nabla_\lambda\tilde{V}_{t}(\theta_{t-1}, \lambda_{t})-\nabla_\lambda \tilde{V}_{t}(\theta_{t-1}, \lambda_{t-1}))( \lambda_{t+1}-\lambda_t)\nn\\
    &=(\nabla_\lambda \tilde{V}_{t+1}(\theta_t, \lambda_t)-\nabla_\lambda\tilde{V}_{t}(\theta_{t-1}, \lambda_{t}))( \lambda_{t+1}-\lambda_t)+(\nabla_\lambda\tilde{V}_{t}(\theta_{t-1}, \lambda_{t})-\nabla_\lambda \tilde{V}_{t}(\theta_{t-1}, \lambda_{t-1}))( \lambda_{t}-\lambda_{t-1})\nn\\
    &\quad+m_{t+1}(\nabla_\lambda\tilde{V}_{t}(\theta_{t-1}, \lambda_{t})-\nabla_\lambda \tilde{V}_{t}(\theta_{t-1}, \lambda_{t-1}) ).
\end{align}
The first term in \eqref{eq:rewritten} can be bounded as
\begin{align}\label{eq:t1}
    &(\nabla_\lambda \tilde{V}_{t+1}(\theta_t, \lambda_t)-\nabla_\lambda\tilde{V}_{t}(\theta_{t-1}, \lambda_{t}))( \lambda_{t+1}-\lambda_t)\nn\\
    &=(\nabla_\lambda {V}^L_\sigma(\theta_t, \lambda_t)-\nabla_\lambda {V}^L_\sigma(\theta_{t-1}, \lambda_{t}))( \lambda_{t+1}-\lambda_t)+ ( b_t\lambda_t-b_{t-1}\lambda_t)( \lambda_{t+1}-\lambda_t)\nn\\
    &\overset{(a)}{\geq} -\frac{1}{2h}(\nabla_\lambda {V}^L_\sigma(\theta_t, \lambda_t)-\nabla_\lambda {V}^L_\sigma(\theta_{t-1}, \lambda_{t}) )^2-\frac{h}{2} (\lambda_{t+1}-\lambda_t)^2\nn\\
    &\quad+ \frac{(b_t-b_{t-1})}{2} (\lambda_{t+1}^2-\lambda_t^2)-\frac{(b_t-b_{t-1})}{2} (\lambda_{t+1}-\lambda_t)^2\nn\\
    &\overset{(b)}{\geq} -\frac{ (C^V_\sigma)^2}{2h}\|\theta_t-\theta_{t-1}\|^2-\frac{h}{2 } (\lambda_{t+1}-\lambda_t)^2\nn\\
    &\quad+ \frac{(b_t-b_{t-1})}{2} (\lambda_{t+1}^2-\lambda_t^2)-\frac{(b_t-b_{t-1})}{2} (\lambda_{t+1}-\lambda_t)^2,
\end{align}
where $(a)$ is from the Cauchy–Schwarz inequality and $(b)$ is from the $C^V_\sigma$-smoothness of $V^L_\sigma$, for any $h>0$.

Similar to \eqref{eq:44}, the second term in \eqref{eq:rewritten}  can be bounded as
\begin{align}\label{eq:t2}
    &(\nabla_\lambda\tilde{V}_{t}(\theta_{t-1}, \lambda_{t})-\nabla_\lambda \tilde{V}_{t}(\theta_{t-1}, \lambda_{t-1}))( \lambda_{t}-\lambda_{t-1})\nn\\
    &\geq \frac{b_{t-1}b_0}{b_{t-1}+b_0}(\lambda_{t}-\lambda_{t-1})^2+ \frac{1}{b_{t-1}+b_0}(\nabla_\lambda \tilde{V}_t(\theta_{t-1},\lambda_{t})-\nabla_\lambda\tilde{V}_t(\theta_{t-1},\lambda_{t-1}) )^2.
\end{align}

The third term in \eqref{eq:rewritten}  can be bounded as 
\begin{align}\label{eq:t3}
    &m_{t+1}(\nabla_\lambda\tilde{V}_{t}(\theta_{t-1}, \lambda_{t})-\nabla_\lambda \tilde{V}_{t}(\theta_{t-1}, \lambda_{t-1}) )\nn\\
    &\geq -\frac{\xi}{2}(\nabla_\lambda\tilde{V}_{t}(\theta_{t-1}, \lambda_{t})-\nabla_\lambda \tilde{V}_{t}(\theta_{t-1}, \lambda_{t-1}))^2-\frac{1}{2\xi} m_{t+1}^2.
\end{align}
Hence combine \eqref{eq:t1} to \eqref{eq:t2} and plug in \eqref{eq:rewritten}, we have that
\begin{align} 
    &(\nabla_\lambda \tilde{V}_{t+1}(\theta_t, \lambda_t)-\nabla_\lambda \tilde{V}_{t}(\theta_{t-1}, \lambda_{t-1}))(\lambda_{t+1}-\lambda_t)\nn\\
    &\geq -\frac{ (C^V_\sigma)^2}{2h}\|\theta_t-\theta_{t-1}\|^2-\frac{h}{2 } (\lambda_{t+1}-\lambda_t)^2\nn\\
    &\quad+ \frac{(b_t-b_{t-1})}{2} (\lambda_{t+1}^2-\lambda_t^2)-\frac{(b_t-b_{t-1})}{2} (\lambda_{t+1}-\lambda_t)^2\nn\\
    &\quad+\frac{b_{t-1}b_0}{b_{t-1}+b_0}(\lambda_{t}-\lambda_{t-1})^2+ \frac{1}{b_{t-1}+b_0}(\nabla_\lambda \tilde{V}_t(\theta_{t-1},\lambda_{t})-\nabla_\lambda\tilde{V}_t(\theta_{t-1},\lambda_{t-1}) )^2\nn\\
    &\quad-\frac{\xi}{2}(\nabla_\lambda\tilde{V}_{t}(\theta_{t-1}, \lambda_{t})-\nabla_\lambda \tilde{V}_{t}(\theta_{t-1}, \lambda_{t-1}))^2-\frac{1}{2\xi} m_{t+1}^2.
\end{align}
Hence \eqref{eq:61} can be further bounded as
\begin{align}
    &(\beta_t m_{t+1})( \lambda_t-\lambda_{t+1})\nn\\
    &\geq ( f(\theta_{t-1})-\hat{f}(\theta_{t-1})-f(\theta_t)+\hat{f}(\theta_t))( -\lambda_t+\lambda_{t+1})\nn\\
    &\quad -\frac{ (C^V_\sigma)^2}{2h}\|\theta_t-\theta_{t-1}\|^2-\frac{h}{2 } (\lambda_{t+1}-\lambda_t)^2\nn\\
    &\quad+ \frac{(b_t-b_{t-1})}{2} (\lambda_{t+1}^2-\lambda_t^2)-\frac{(b_t-b_{t-1})}{2} (\lambda_{t+1}-\lambda_t)^2\nn\\
    &\quad+\frac{b_{t-1}b_0}{b_{t-1}+b_0}(\lambda_{t}-\lambda_{t-1})^2+ \frac{1}{b_{t-1}+b_0}(\nabla_\lambda \tilde{V}_t(\theta_{t-1},\lambda_{t})-\nabla_\lambda\tilde{V}_t(\theta_{t-1},\lambda_{t-1}) )^2\nn\\
    &\quad-\frac{\xi}{2}(\nabla_\lambda\tilde{V}_{t}(\theta_{t-1}, \lambda_{t})-\nabla_\lambda \tilde{V}_{t}(\theta_{t-1}, \lambda_{t-1}))^2-\frac{1}{2\xi} m_{t+1}^2.
\end{align}
It can be directly verified that 
\begin{align}
    m_{t+1}(\lambda_t-\lambda_{t+1}) = \frac{1}{2}(\lambda_t-\lambda_{t-1} )^2-\frac{1}{2}(\lambda_t-\lambda_{t+1} )^2- \frac{m_{t+1}^2}{2}.
\end{align}
Recall that $\beta_t=\frac{1}{\xi}$, hence 
\begin{align}
    &\frac{1}{2\xi}(\lambda_t-\lambda_{t-1} )^2-\frac{1}{2\xi}(\lambda_t-\lambda_{t+1} )^2- \frac{m_{t+1}^2}{2\xi}\nn\\
    &\geq( f(\theta_{t-1})-\hat{f}(\theta_{t-1})-f(\theta_t)+\hat{f}(\theta_t))( -\lambda_t+\lambda_{t+1})\nn\\
    &\quad -\frac{ (C^V_\sigma)^2}{2h}\|\theta_t-\theta_{t-1}\|^2-\frac{h}{2 } (\lambda_{t+1}-\lambda_t)^2\nn\\
    &\quad+ \frac{(b_t-b_{t-1})}{2} (\lambda_{t+1}^2-\lambda_t^2)-\frac{(b_t-b_{t-1})}{2} (\lambda_{t+1}-\lambda_t)^2\nn\\
    &\quad+\frac{b_{t-1}b_0}{b_{t-1}+b_0}(\lambda_{t}-\lambda_{t-1})^2+ \frac{1}{b_{t-1}+b_0}(\nabla_\lambda \tilde{V}_t(\theta_{t-1},\lambda_{t})-\nabla_\lambda\tilde{V}_t(\theta_{t-1},\lambda_{t-1}) )^2\nn\\
    &\quad-\frac{\xi}{2}(\nabla_\lambda\tilde{V}_{t}(\theta_{t-1}, \lambda_{t})-\nabla_\lambda \tilde{V}_{t}(\theta_{t-1}, \lambda_{t-1}))^2-\frac{1}{2\xi} m_{t+1}^2.
\end{align}
From $\xi\leq \frac{1}{b_0}\leq \frac{2}{b_0+b_{t-1}}$, we have  $\frac{1}{b_{t-1}+b_0}(\nabla_\lambda \tilde{V}_t(\theta_{t-1},\lambda_{t})-\nabla_\lambda\tilde{V}_t(\theta_{t-1},\lambda_{t-1}) )^2 -\frac{\xi}{2}(\nabla_\lambda\tilde{V}_{t}(\theta_{t-1}, \lambda_{t})-\nabla_\lambda \tilde{V}_{t}(\theta_{t-1}, \lambda_{t-1}))^2\geq 0$. Also, it can be shown that $\frac{b_{t-1}b_0}{b_{t-1}+b_0}\geq \frac{b_{t-1}b_0}{2b_0}=\frac{b_{t-1}}{2}$.
Thus, it follows that 
\begin{align}
    &\frac{1}{2\xi}(\lambda_t-\lambda_{t-1} )^2-\frac{1}{2\xi}(\lambda_t-\lambda_{t+1} )^2- \frac{m_{t+1}^2}{2\xi}\nn\\
    &\geq( f(\theta_{t-1})-\hat{f}(\theta_{t-1})-f(\theta_t)+\hat{f}(\theta_t))( -\lambda_t+\lambda_{t+1})\nn\\
    &\quad -\frac{ (C^V_\sigma)^2}{2h}\|\theta_t-\theta_{t-1}\|^2-\frac{h}{2 } (\lambda_{t+1}-\lambda_t)^2\nn\\
    &\quad+ \frac{(b_t-b_{t-1})}{2} (\lambda_{t+1}^2-\lambda_t^2)-\frac{(b_t-b_{t-1})}{2} (\lambda_{t+1}-\lambda_t)^2\nn\\
    &\quad+\frac{b_{t-1}}{2}(\lambda_{t}-\lambda_{t-1})^2-\frac{1}{2\xi} m_{t+1}^2.
\end{align}
Re-arrange the terms, it follows that
\begin{align}
    &-\frac{1}{2\xi}(\lambda_t-\lambda_{t+1} )^2-\frac{b_t-b_{t-1}}{2}\lambda_{t+1}^2\nn\\
    &\geq ( f(\theta_{t-1})-\hat{f}(\theta_{t-1})-f(\theta_t)+\hat{f}(\theta_t))( -\lambda_t+\lambda_{t+1}) -\frac{ (C^V_\sigma)^2}{2h}\|\theta_t-\theta_{t-1}\|^2-\frac{h}{2 } (\lambda_{t+1}-\lambda_t)^2\nn\\
    &\quad-\frac{(b_t-b_{t-1})}{2}  \lambda_t^2 -\frac{(b_t-b_{t-1})}{2} (\lambda_{t+1}-\lambda_t)^2+\frac{b_{t-1}}{2}(\lambda_{t}-\lambda_{t-1})^2-\frac{1}{2\xi}(\lambda_t-\lambda_{t-1} )^2\nn\\
    &\geq -\frac{1}{2\xi}(\lambda_t-\lambda_{t-1} )^2-\frac{(b_t-b_{t-1})}{2}  \lambda_t^2+( f(\theta_{t-1})-\hat{f}(\theta_{t-1})-f(\theta_t)+\hat{f}(\theta_t))( -\lambda_t+\lambda_{t+1})\nn\\
    &\quad -\frac{ (C^V_\sigma)^2}{2h}\|\theta_t-\theta_{t-1}\|^2-\frac{h}{2 } (\lambda_{t+1}-\lambda_t)^2+\frac{b_{t-1}}{2}(\lambda_{t}-\lambda_{t-1})^2,
\end{align}
where the last inequality is from the fact that $b_t$ is decreasing. 

Now multiply $\frac{2}{\xi b_t}$ on both sides, we further have that
\begin{align}\label{eq:102}
    &-\frac{1}{\xi^2b_t}(\lambda_t-\lambda_{t+1} )^2-\frac{1}{\xi}\left(1-\frac{b_{t-1}}{b_t}\right)\lambda_{t+1}^2\nn\\
    &\geq -\frac{1}{\xi^2b_t}(\lambda_t-\lambda_{t-1} )^2-\frac{1}{\xi}\left(1-\frac{b_{t-1}}{b_t}\right) \lambda_t^2+\frac{2}{\xi b_t}( f(\theta_{t-1})-\hat{f}(\theta_{t-1})-f(\theta_t)+\hat{f}(\theta_t))( -\lambda_t+\lambda_{t+1})\nn\\
    &\quad -\frac{ (C^V_\sigma)^2}{h\xi b_t}\|\theta_t-\theta_{t-1}\|^2-\frac{h}{\xi b_t} (\lambda_{t+1}-\lambda_t)^2+\frac{1}{\xi}(\lambda_{t}-\lambda_{t-1})^2.
\end{align}
If we set $h=\frac{b_t}{2}$, \eqref{eq:102} can be rewritten as
\begin{align}
    &-\frac{1}{\xi^2b_t}(\lambda_t-\lambda_{t+1} )^2-\frac{1}{\xi}\left(1-\frac{b_{t-1}}{b_t}\right)\lambda_{t+1}^2\nn\\
    &\geq -\frac{1}{\xi^2b_t}(\lambda_t-\lambda_{t-1} )^2-\frac{1}{\xi}\left(1-\frac{b_{t-1}}{b_t}\right) \lambda_t^2+\frac{2}{\xi b_t}( f(\theta_{t-1})-\hat{f}(\theta_{t-1})-f(\theta_t)+\hat{f}(\theta_t))( -\lambda_t+\lambda_{t+1})\nn\\
    &\quad -\frac{ 2(C^V_\sigma)^2}{ \xi b^2_t}\|\theta_t-\theta_{t-1}\|^2-\frac{1}{2\xi } (\lambda_{t+1}-\lambda_t)^2+\frac{1}{\xi}(\lambda_{t}-\lambda_{t-1})^2.
\end{align}
Further we have that 
\begin{align}\label{eq:104}
    &-\frac{1}{\xi^2b_{t+1}}(\lambda_t-\lambda_{t+1} )^2+\left(\frac{1}{\xi^2b_{t+1}}-\frac{1}{\xi^2b_{t}} \right)(\lambda_t-\lambda_{t+1} )^2-\frac{1}{\xi}\left(1-\frac{b_{t}}{b_{t+1}}\right)\lambda_{t+1}^2+ \frac{1}{\xi}\left(\frac{b_{t-1}}{b_t}-\frac{b_{t}}{b_{t+1}}\right)\lambda_{t+1}^2\nn\\
    &\geq -\frac{1}{\xi^2b_t}(\lambda_t-\lambda_{t-1} )^2-\frac{1}{\xi}\left(1-\frac{b_{t-1}}{b_t}\right) \lambda_t^2+\frac{2}{\xi b_t}( f(\theta_{t-1})-\hat{f}(\theta_{t-1})-f(\theta_t)+\hat{f}(\theta_t))( -\lambda_t+\lambda_{t+1})\nn\\
    &\quad -\frac{ 2(C^V_\sigma)^2}{ \xi b^2_t}\|\theta_t-\theta_{t-1}\|^2-\frac{1}{2\xi } (\lambda_{t+1}-\lambda_t)^2+\frac{1}{\xi}(\lambda_{t}-\lambda_{t-1})^2.
\end{align}
Re-arranging the terms in \eqref{eq:104} implies that
\begin{align}
    &-\frac{1}{\xi^2b_{t+1}}(\lambda_t-\lambda_{t+1} )^2-\frac{1}{\xi}\left(1-\frac{b_{t}}{b_{t+1}}\right)\lambda_{t+1}^2-\left(-\frac{1}{\xi^2b_t}(\lambda_t-\lambda_{t-1} )^2-\frac{1}{\xi}\left(1-\frac{b_{t-1}}{b_t}\right) \lambda_t^2 \right)\nn\\
    &\geq -\left(\frac{1}{\xi^2b_{t+1}}-\frac{1}{\xi^2b_{t}} \right)(\lambda_t-\lambda_{t+1} )^2 -\frac{1}{\xi}\left(\frac{b_{t-1}}{b_t}-\frac{b_{t}}{b_{t+1}}\right)\lambda_{t+1}^2\nn\\
    &\quad -\frac{ 2(C^V_\sigma)^2}{ \xi b^2_t}\|\theta_t-\theta_{t-1}\|^2-\frac{1}{2\xi } (\lambda_{t+1}-\lambda_t)^2+\frac{1}{\xi}(\lambda_{t}-\lambda_{t-1})^2\nn\\
    &\quad +\frac{2}{\xi b_t}( f(\theta_{t-1})-\hat{f}(\theta_{t-1})-f(\theta_t)+\hat{f}(\theta_t))( -\lambda_t+\lambda_{t+1})\nn\\
    &\geq -\frac{7}{10\xi} (-\lambda_t+\lambda_{t+1})^2 -\frac{ 2(C^V_\sigma)^2}{ \xi b^2_t}\|\theta_t-\theta_{t-1}\|^2+\frac{1}{\xi}(\lambda_{t}-\lambda_{t-1})^2+\frac{1}{\xi}\left(\frac{b_{t}}{b_{t+1}}-\frac{b_{t-1}}{b_t}\right)\lambda_{t+1}^2\nn\\
    &\quad +\frac{2}{\xi b_t}( f(\theta_{t-1})-\hat{f}(\theta_{t-1})-f(\theta_t)+\hat{f}(\theta_t))( -\lambda_t+\lambda_{t+1}),
\end{align}
where the last inequality is from $\frac{1}{b_{t+1}}-\frac{1}{b_t}\leq \frac{\xi}{5}$. Recall in Lemma \ref{lemma5}, we showed that 
\begin{align}
     &V^L_{\sigma}(\theta_{t+1},\lambda_{t+1})-V^L_{\sigma}(\theta_{t},\lambda_{t}) \nn\\
     &\geq ( f(\theta_{t-1})-\hat{f}(\theta_{t-1}))( \lambda_{t+1}-\lambda_t )+\langle\theta_{t+1}-\theta_t, -\hat{g}(\theta_{t},\lambda_{t+1})+ g(\theta_{t},\lambda_{t+1}) \rangle  - \frac{\xi (C_\sigma^V)^2}{2} \|\theta_t-\theta_{t-1}\|^2\nn\\
     &\quad+\left(  \frac{\mu_t}{2}+\nu  \right)\|\theta_{t+1}-\theta_t\|^2+\frac{b_{t-1}}{2}(\lambda_t^2-\lambda_{t+1}^2)-\frac{1}{\xi}(\lambda_{t+1}-\lambda_t)^2-\frac{1}{2\xi}(\lambda_t-\lambda_{t-1})^2.
\end{align}
Combine both inequality together, and we further have that
\begin{align}
    &-\frac{8}{\xi^2b_{t+1}}(\lambda_t-\lambda_{t+1} )^2-\frac{8}{\xi}\left(1-\frac{b_{t}}{b_{t+1}}\right)\lambda_{t+1}^2-\left(-\frac{8}{\xi^2b_t}(\lambda_t-\lambda_{t-1} )^2-\frac{8}{\xi}\left(1-\frac{b_{t-1}}{b_t}\right) \lambda_t^2 \right)\nn\\
    &\quad + V^L_{\sigma}(\theta_{t+1},\lambda_{t+1})-V^L_{\sigma}(\theta_{t},\lambda_{t}) \nn\\
    &\geq -\frac{28}{5\xi} (-\lambda_t+\lambda_{t+1})^2 -\frac{ 16(C^V_\sigma)^2}{ \xi b^2_t}\|\theta_t-\theta_{t-1}\|^2+\frac{8}{\xi}(\lambda_{t}-\lambda_{t-1})^2+\frac{8}{\xi}\left(\frac{b_{t}}{b_{t+1}}-\frac{b_{t-1}}{b_t}\right)\lambda_{t+1}^2\nn\\
    &\quad +\frac{16}{b_t\xi}( f(\theta_{t-1})-\hat{f}(\theta_{t-1})-f(\theta_t)+\hat{f}(\theta_t))( -\lambda_t+\lambda_{t+1})\nn\\
    &\quad+ ( f(\theta_{t-1})-\hat{f}(\theta_{t-1}))( \lambda_{t+1}-\lambda_t )+\langle\theta_{t+1}-\theta_t, -\hat{g}(\theta_{t},\lambda_{t+1})+ g(\theta_{t},\lambda_{t+1}) \rangle  - \frac{\xi (C_\sigma^V)^2}{2} \|\theta_t-\theta_{t-1}\|^2\nn\\
     &\quad+\left(  \frac{\mu_t}{2}+\nu  \right)\|\theta_{t+1}-\theta_t\|^2+\frac{b_{t-1}}{2}(\lambda_t^2-\lambda_{t+1}^2)-\frac{1}{\xi}(\lambda_{t+1}-\lambda_t)^2-\frac{1}{2\xi}(\lambda_t-\lambda_{t-1})^2\nn\\
     &=S_t+\left(-\frac{ 16(C^V_\sigma)^2}{ \xi b^2_t}- \frac{\xi (C_\sigma^V)^2}{2} \right)\|\theta_t-\theta_{t-1}\|^2+\left(-\frac{28}{5\xi}-\frac{1}{\xi} \right)(-\lambda_t+\lambda_{t+1})^2+\frac{b_{t-1}}{2}(\lambda_t^2-\lambda_{t+1}^2)\nn\\
    &\quad+\left(\frac{8}{\xi}-\frac{1}{2\xi}\right)(\lambda_{t}-\lambda_{t-1})^2+\frac{8}{\xi}\left(\frac{b_{t}}{b_{t+1}}-\frac{b_{t-1}}{b_t}\right)\lambda_{t+1}^2+\left(  \frac{\mu_t}{2}+\nu  \right)\|\theta_{t+1}-\theta_t\|^2,
\end{align}
where $S_t\triangleq \frac{16}{b_t\xi}( f(\theta_{t-1})-\hat{f}(\theta_{t-1})-f(\theta_t)+\hat{f}(\theta_t))( -\lambda_t+\lambda_{t+1})+ ( f(\theta_{t-1})-\hat{f}(\theta_{t-1}))( \lambda_{t+1}-\lambda_t )+\langle\theta_{t+1}-\theta_t, -\hat{g}(\theta_{t},\lambda_{t+1})+ g(\theta_{t},\lambda_{t+1}) \rangle$.
Now 
\begin{align}\label{eq:80}
    &-\frac{8}{\xi^2b_{t+1}}(\lambda_t-\lambda_{t+1} )^2-\frac{8}{\xi}\left(1-\frac{b_{t}}{b_{t+1}}\right)\lambda_{t+1}^2-\left(-\frac{8}{\xi^2b_t}(\lambda_t-\lambda_{t-1} )^2-\frac{8}{\xi}\left(1-\frac{b_{t-1}}{b_t}\right) \lambda_t^2 \right)\nn\\
    &\quad + V^L_{\sigma}(\theta_{t+1},\lambda_{t+1})-V^L_{\sigma}(\theta_{t},\lambda_{t})+ \frac{b_t}{2}\lambda_{t+1}^2 -\frac{b_{t-1}}{2}\lambda_{t}^2 \nn\\
    &\quad +\left(-\frac{ 16(C^V_\sigma)^2}{ \xi b^2_{t+1}}- \frac{\xi (C_\sigma^V)^2}{2} \right)\|\theta_{t+1}-\theta_{t}\|^2 -\left(-\frac{ 16(C^V_\sigma)^2}{ \xi b^2_t}- \frac{\xi (C_\sigma^V)^2}{2} \right)\|\theta_t-\theta_{t-1}\|^2\nn\\
    &\quad +\left(\frac{8}{\xi}-\frac{1}{2\xi}\right)(\lambda_{t+1}-\lambda_{t})^2- \left(\frac{8}{\xi}-\frac{1}{2\xi}\right)(\lambda_{t}-\lambda_{t-1})^2\nn\\
    &\geq S_t+\left(\frac{\mu_t}{2}+\nu -\frac{ 16(C^V_\sigma)^2}{ \xi b^2_{t+1}}- \frac{\xi (C_\sigma^V)^2}{2} \right)\|\theta_{t+1}-\theta_{t}\|^2 + \frac{b_t-b_{t-1}}{2}\lambda_{t+1}^2\nn\\
    &\quad+\left(\frac{8}{\xi}-\frac{1}{2\xi}-\frac{28}{5\xi}-\frac{1}{\xi}\right)(\lambda_{t+1}-\lambda_{t})^2+\frac{8}{\xi}\left(\frac{b_{t}}{b_{t+1}}-\frac{b_{t-1}}{b_t}\right)\lambda_{t+1}^2\nn\\
    &=S_t+\left(\frac{\mu_t}{2}+\nu -\frac{ 16(C^V_\sigma)^2}{ \xi b^2_{t+1}}- \frac{\xi (C_\sigma^V)^2}{2} \right)\|\theta_{t+1}-\theta_{t}\|^2 + \frac{b_t-b_{t-1}}{2}\lambda_{t+1}^2\nn\\
    &\quad+\frac{9}{10\xi}(\lambda_{t+1}-\lambda_{t})^2+\frac{8}{\xi}\left(\frac{b_{t}}{b_{t+1}}-\frac{b_{t-1}}{b_t}\right)\lambda_{t+1}^2,
\end{align}
which then completes the proof. 
\end{proof}

We now restate Theorem \ref{thm2} with all the specific step sizes. The definitions of these constants can also be found in Section \ref{sec:constants}.
\begin{theorem}(Restatement of Theorem \ref{thm2})
Set $b_t=\frac{19}{20\xi t^{0.25}}$, $\mu_t=\xi (C_\sigma^V)^2+\frac{16\tau(C_\sigma^V)^2 }{\xi (b_{t+1})^2}-2\nu,$
    $\beta_t=\frac{1}{\xi},
    \alpha_t=\nu+\mu_t$,  where $\xi>\frac{2\nu+(1+\Lambda^*)L_\sigma}{(C^V_\sigma)^2}$, $\nu$ is any positive number and $\tau$ is any number greater than $2$. Moreover, set $\epe=\frac{1}{t^{0.5}L_\Omega}\frac{1}{32t^{0.25}\Lambda^*+2\Lambda^*+\frac{1}{\alpha_1}(1+\Lambda^*)C^V_\sigma }\frac{19^2\epsilon^2}{3200\xi(\tau-2)(C^V_\sigma)^2uL_\Omega}=\mathcal{O}(\frac{\epsilon^2}{t^{0.75}})$, 
    then \begin{align}
    \min_{1\leq t\leq T}\|G_t\|^2\leq (1+\sqrt{2})\epsilon,
\end{align}
when $T=\mathcal{O}(\epsilon^{-4})$.
\end{theorem} 
\begin{proof}
Denote by  $p_t\triangleq\frac{8(\tau-2)(C_\sigma^V)^2 }{\xi b_{t+1}^2}$ and $M_1\triangleq\frac{16\tau^2}{(\tau-2)^2}+\frac{(\xi (C^V_\sigma)^2-\nu)^2}{64(\tau-2)^2(C^V_\sigma)^2\xi^2}$. Then it can be verified that  $\nu+\frac{\mu_t}{2}-\frac{\xi (C_\sigma^V)^2 }{2}-\frac{16 (C_\sigma^V)^2 }{\xi b_{t+1}^2}=p_t$.
Then \eqref{eq:80} can be rewritten as 
\begin{align}\label{eq:82}
    F_{t+1}-F_t&\geq S_t+ p_t\|\theta_{t+1}-\theta_t\|^2+ \frac{b_t-b_{t-1}}{2}\lambda_{t+1}^2\nn\\
    &\quad+\frac{9}{10\xi}(\lambda_{t+1}-\lambda_{t})^2+\frac{8}{\xi}\left(\frac{b_{t}}{b_{t+1}}-\frac{b_{t-1}}{b_t}\right)\lambda_{t+1}^2.
\end{align}
From the definition, we have that
\begin{align}
G_t=\left[
\begin{array}{lr}
     &\beta_t \left(\lambda_t-\mathbf{\prod}_{[0,\Lambda^*]} \left(\lambda_t -\frac{1}{\beta_t} \left(\nabla_\lambda V^L_\sigma (\theta_t,\lambda_t) \right)\right)\right)\\
     &\alpha_t\left(\theta_t-\mathbf{\prod}_\Theta \left(\theta_t+\frac{1}{\alpha_t} \left(\nabla_\theta V^L_\sigma (\theta_t,\lambda_t) \right) \right)\right)
\end{array}
\right],
\end{align}
and denote by
\begin{align}
\tilde{G}_t\triangleq\left[
\begin{array}{lr}
     &\beta_t \left(\lambda_t-\mathbf{\prod}_{[0,\Lambda^*]} \left(\lambda_t -\frac{1}{\beta_t} \left( \nabla_\lambda \tilde{V_t} (\theta_t,\lambda_t)\right) \right)\right)\\
     &\alpha_t\left(\theta_t-\mathbf{\prod}_\Theta \left(\theta_t+\frac{1}{\alpha_t} \left(\nabla_\theta\tilde{V_t}(\theta_t,\lambda_t)  \right) \right)\right)
\end{array}
\right].
\end{align}
It can be verified that
\begin{align}\label{eq:113}
    \|G_t\|-\| \tilde{G}_t\| \leq b_{t-1}|\lambda_t|.
\end{align}
From Theorem 4.2 in \cite{xu2020unified}, it can be shown that
\begin{align}
    \| \tilde{G}_t\|^2\leq 2(\mu_t+\nu)^2\|\theta_{t+1}-\theta_t\|^2+\left(2(C^V_\sigma)^2+\frac{1}{\xi^2} \right) (\lambda_{t+1}-\lambda_t)^2,
\end{align}
and 
\begin{align}
    M_1\geq \frac{2(\nu+\mu_t)^2}{p_t^2}. 
\end{align}
Hence
\begin{align}
    \| \tilde{G}_t\|^2\leq M_1 p_t^2\|\theta_{t+1}-\theta_t\|^2+\left(2(C^V_\sigma)^2+\frac{1}{\xi^2} \right) (\lambda_{t+1}-\lambda_t)^2.
\end{align}
Set $u_t\triangleq \frac{1}{\max\left\{ M_1p_t, \frac{10+20\xi^2(C^V_\sigma)^2}{9\xi}  \right\}}$, then from \eqref{eq:82}, we have that
\begin{align}
    u_t\| \tilde{G}_t\|^2\leq F_{t+1}-F_t -S_t-\frac{b_t-b_{t-1}}{2}\lambda_{t+1}^2-\frac{8}{\xi}\left(\frac{b_{t}}{b_{t+1}}-\frac{b_{t-1}}{b_t}\right)\lambda_{t+1}^2.
\end{align}
Summing the inequality above from $t=1$ to $T$, then 
\begin{align}
    \sum^T_{t=1}u_t\| \tilde{G}_t\|^2&\leq F_{T+1}-F_1-\sum^T_{t=1}S_t+\frac{8}{\xi}\left(\frac{b_0}{b_1}\lambda_2^2 -\frac{b_T}{b_{T+1}}\lambda_{T+1}^2\right)+\left( \frac{b_0-b_T}{2}(\Lambda^*)^2 \right)\nn\\
    &\leq F_{T+1}-F_1-\sum^T_{t=1}S_t+\frac{8}{\xi}\frac{b_0}{b_1}(\Lambda^*)^2+\left( \frac{b_0-b_T}{2}(\Lambda^*)^2 \right),
\end{align}
which is from $b_t$ is decreasing and $\lambda_t<\Lambda^*$.
Note that 
\begin{align}
    \max_{t\geq 1}\max_{\theta\in\Theta,\lambda\in[0,\Lambda^*]} F_t &= \max\bigg\{-\frac{8}{\xi^2b_{t+1}}(\lambda_t-\lambda_{t+1} )^2-\frac{8}{\xi}\left(1-\frac{b_{t}}{b_{t+1}}\right)\lambda_{t+1}^2 + V^L_{\sigma}(\theta_{t+1},\lambda_{t+1})+ \frac{b_t}{2}\lambda_{t+1}^2 \nn\\
    &\quad +\left(-\frac{ 16(C^V_\sigma)^2}{ \xi b^2_{t+1}}- \frac{\xi (C_\sigma^V)^2}{2} \right)\|\theta_{t+1}-\theta_{t}\|^2   +\left(\frac{8}{\xi}-\frac{1}{2\xi}\right)(\lambda_{t+1}-\lambda_{t})^2\bigg\}\nn\\
    &\leq \frac{1.6}{\xi}(\Lambda^*)^2+(1+\Lambda^*)(2C_\sigma)+\frac{b_1}{2}(\Lambda^*)^2+\frac{15}{2\xi}(\Lambda^*)^2\nn\\
    &\triangleq F^*,
\end{align}
which is from the definition of $b_t$, and $8(\frac{b_t}{b_{t+1}}-1)\leq 8 (\frac{(t+1)^{0.25}}{t^{0.25}}-1)\leq 8 (\frac{2^{0.25}}{1}-1)<1.6$.
Then plugging in the definition of $b_t$ implies that
\begin{align}
    \sum^T_{t=1}u_t\| \tilde{G}_t\|^2& \leq F^*-F_1-\sum^T_{t=1}S_t+\frac{8}{\xi}(\Lambda^*)^2+\left( \frac{b_0}{2}(\Lambda^*)^2 \right).
\end{align}
If moreover set $u\triangleq\max\left\{M_1, \frac{10+20\xi^2(C^V_\sigma)^2}{9\xi p_2} \right\}$, then $u_t\geq \frac{1}{up_t}$, and hence
\begin{align}
    \frac{\sum^T_{t=1}\frac{1}{p_t}\| \tilde{G}_t\|^2}{\sum^T_{t=1}\frac{1}{p_t}}& \leq \frac{u}{\sum^T_{t=1}\frac{1}{p_t}}\left( F^*-F_1-\sum^T_{t=1}S_t+\frac{8}{\xi}(\Lambda^*)^2+\left( \frac{b_0}{2}(\Lambda^*)^2 \right)\right).
\end{align}
Plug in the definition of $p_t$ then we have that
\begin{align}\label{eq:93}
    \frac{\sum^T_{t=1}\frac{1}{p_t}\| \tilde{G}_t\|^2}{\sum^T_{t=1}\frac{1}{p_t}}& \leq \frac{3200\xi(\tau-2)(C^V_\sigma)^2d}{19^2(\sqrt{T}-2)}\left( F^*-F_1-\sum^T_{t=1}S_t+\frac{8}{\xi}(\Lambda^*)^2+\left( \frac{b_0}{2}(\Lambda^*)^2 \right)\right).
\end{align}
We moreover have that  
\begin{align}
    |S_t|&=\bigg|\frac{16}{b_t\xi}( f(\theta_{t-1})-\hat{f}(\theta_{t-1})-f(\theta_t)+\hat{f}(\theta_t))( -\lambda_t+\lambda_{t+1})+ ( f(\theta_{t-1})-\hat{f}(\theta_{t-1}))( \lambda_{t+1}-\lambda_t )\nn\\
    &\quad+\langle\theta_{t+1}-\theta_t, -\hat{g}(\theta_{t},\lambda_{t+1})+ g(\theta_{t},\lambda_{t+1}) \rangle\bigg|\nn\\
    &\leq  32t^{0.25} \Lambda^*(\Omega_{t-1}+\Omega_{t}) + 2\Lambda^*\Omega_{t-1}+\frac{1}{\alpha_t}(1+\Lambda^*)C^V_\sigma \Omega_{t},
\end{align}
where $\Omega_t\triangleq \max\left\{ \|g(\theta_t,\lambda_{t+1})-\hat{g}(\theta_t,\lambda_{t+1})\|,|f(\theta_t)-\hat{f}(\theta_t)| \right\}$. Note that it has been shown in \cite{wang2022policy} that $\Omega_t\leq L_\Omega \max\left\{ \|Q_{\sigma,r}-\hat{Q}_{\sigma,r} \|,\|Q_{\sigma,c}-\hat{Q}_{\sigma,c} \| \right\}=L_\Omega \epe$, and hence $\Omega_t$ can be controlled by setting $\epe$.

Note that $\alpha_t=\nu+\mu_t$ is increasing, hence $\frac{1}{\alpha_t}\leq \frac{1}{\alpha_1}$. Hence if we set $\epe=\frac{1}{t^{0.5}L_\Omega}\frac{1}{32t^{0.25}\Lambda^*+2\Lambda^*+\frac{1}{\alpha_1}(1+\Lambda^*)C^V_\sigma }\frac{19^2\epsilon^2}{3200\xi(\tau-2)(C^V_\sigma)^2uL_\Omega}=\mathcal{O}(\frac{\epsilon^2}{t^{0.75}})$, then 
\begin{align}
    |S_t| \leq  \frac{1}{t^{0.5}}\frac{19^2\epsilon^2}{3200\xi(\tau-2)(C^V_\sigma)^2uL_\Omega},
\end{align}
and hence
\begin{align}
    \left|\sum^T_{t=1}S_t \right| \leq  \sqrt{T}\frac{19^2\epsilon^2}{3200\xi(\tau-2)(C^V_\sigma)^2uL_\Omega}.
\end{align}\
Thus plug in \eqref{eq:93} and we have that
\begin{align} 
    \frac{\sum^T_{t=1}\frac{1}{p_t}\| \tilde{G}_t\|^2}{\sum^T_{t=1}\frac{1}{p_t}}& \leq \frac{3200\xi(\tau-2)(C^V_\sigma)^2u}{19^2(\sqrt{T}-2)} K +\epsilon^2,
\end{align}
where $K=F^*-F_1+\frac{8}{\xi}(\Lambda^*)^2+\left( \frac{b_1}{2}(\Lambda^*)^2 \right)$.
When $T=(2+\frac{3200\xi(\tau-2)(C^V_\sigma)^2uK}{19^2\epsilon^2})^2$, we have that 
\begin{align} 
    \frac{\sum^T_{t=1}\frac{1}{p_t}\| \tilde{G}_t\|^2}{\sum^T_{t=1}\frac{1}{p_t}}\leq 2\epsilon^2.
\end{align}
Similarly to Theorem 4.2 in \cite{xu2020unified}, if $t>\frac{19^4(\Lambda^*)^4}{2\cdot10^4\xi^4\epsilon^4}$, then $b_{t-1}<\frac{\epsilon}{\Lambda^*}$ and $b_{t-1}\lambda_t<\epsilon$. 
Hence combine with \eqref{eq:113} we finally have that
\begin{align}
    \min_{1\leq t\leq T}\| G_t\|\leq (1+\sqrt{2})\epsilon,
\end{align}
when $T=\max\left\{ \frac{7(\Lambda^*)^4}{\xi^4\epsilon^4},\left(2+\frac{9\xi(\tau-2)(C^V_\sigma)^2uK}{\epsilon^2}\right)^2\right\}=\mathcal{O}(\epsilon^{-4})$.
\end{proof}
\begin{remark}
Note that the sample complexity of robust TD algorithm to achieve an $\epe$-error bound is $\mathcal{O}(\epe^{-2})$, hence the sample complexity at the time step $t$ is $\mathcal{O}(\epe^{-2})=\mathcal{O}(\frac{t^{1.5}}{\epsilon^4})$. Thus the total sample complexity to find an $\epsilon$-stationary solution is $\sum^T_{t=1} \frac{t^{1.5}}{\epsilon^4}=\mathcal{O}(\epsilon^{-14})$. This great increasing of complexity is due to the estimation of robust value functions. 
\end{remark}

\section{Constants}\label{sec:constants}
In this section, we summarize the definitions of all the constants we used in this paper. 
\begin{align}
    L_V&=\frac{k |\mca|}{(1-\gamma)^2},\nn\\
    C_\sigma&=\frac{1}{1-\gamma}(1+2\gamma \delta \frac{\log|\mcs|}{\sigma}),\nn\\
    C^V_\sigma&=\frac{1}{1-\gamma} |\mca|k C_\sigma,\nn\\
    k_B&=\frac{1}{1-\gamma +\gamma \delta} \left( |\mca|C_\sigma l+|\mca|k C^V_\sigma\right) +\frac{2|\mca|^2\gamma (1-\delta)}{(1-\gamma+\gamma \delta)^2}k^2C_\sigma,\nn\\
    L_\sigma&=k_B+\frac{\gamma \delta}{1-\gamma}\left(\sqrt{|\mcs|}k_B+ 2\sigma|\mcs|C^V_\sigma \frac{1}{1-\gamma+\gamma \delta}k|\mca|C_\sigma\right),\nn\\
    b_t&=\frac{19}{20\xi t^{0.25}},\nn\\ M_1&=\frac{16\tau^2}{(\tau-2)^2}+\frac{(\xi (C^V_\sigma)^2-\nu)^2}{64(\tau-2)^2(C^V_\sigma)^2\xi^2},\nn\\
    u&=\max\left\{M_1, \frac{10+20\xi^2(C^V_\sigma)^2}{9\xi p_2} \right\},\nn\\
    F^*&= \frac{1.6}{\xi}(\Lambda^*)^2+(1+\Lambda^*)(2C_\sigma)+\frac{b_1}{2}(\Lambda^*)^2+\frac{15}{2\xi}(\Lambda^*)^2,\nn\\
    K&=F^*-F_1+\frac{8}{\xi}(\Lambda^*)^2+\left( \frac{b_1}{2}(\Lambda^*)^2 \right),\nn\\
    \mu_t&=\xi (C_\sigma^V)^2+\frac{16\tau(C_\sigma^V)^2 }{\xi (b_{t+1})^2}-2\nu,\nn\\
    \beta_t&=\frac{1}{\xi},\nn\\
    \alpha_t&=\nu+\mu_t.
\end{align}

\end{document}